%% file: jmlr_approximation_paper.tex
\documentclass[twoside,11pt]{article}

\usepackage{jmlr2e}
\hypersetup{
    colorlinks=false, 
    linktoc=all,     
    linkcolor=blue,  
}
\usepackage{array}

\usepackage{lastpage}

\usepackage{graphicx}
\usepackage[margin=1in]{geometry}
\usepackage[utf8]{inputenc}
\usepackage{amsmath}
\usepackage{amsfonts}
\usepackage{amssymb, bbold}
\usepackage{xcolor}
\usepackage{xfrac}
\usepackage{enumitem}
\usepackage{tikz}
\usepackage{bbm}
\usetikzlibrary{fit,positioning}

\usepackage{color}   
\usepackage[capitalise]{cleveref}
\newcommand{\M}{\mathcal{M}}

\newcommand{\RR}{\mathbb{R}}

\newcommand{\NN}{\mathbb{N}}

\usepackage{bm} 

\newtheorem{definition}{Definition}

\newcommand{\cT}{\mathcal{T}}
\newcommand{\textdef}[1]{\textnormal{\textbf{#1}}}

\newcommand{\cent}{\mathrm{cent}}

\newcommand{\ah}{$\alpha$-H\"older }

\newcommand{\dimM}{\overline{\mathrm{dim}}}

\newcommand{\dint}{d_{\text{intrinsic}}}

\ShortHeadings{Approximation Rates for Group Equivariant Learning}{Siegel, Hordan, Lawrence, Syed, and Dym}
\firstpageno{1}

\begin{document}

\title{Quantitative Approximation Rates for Group Equivariant Learning}

\author{\name Jonathan W. Siegel \email jwsiegel@tamu.edu\\
\addr Department of Mathematics \\
      Texas A\&M University\\
      College Station, TX, 77840, USA 
       \AND
        \name Snir Hordan \email snirhordan@campus.technion.ac.il \\
       \addr Faculty of Mathematics\\
      Technion- Israel Institute of Technology\\
      Haifa, 3200003,  Israel
      \AND
      \name Hannah Lawrence \email hanlaw@mit.edu\\
      \addr Computer Science and Artificial Intelligence Laboratory \\
      Massachusetts Institute of Technology\\
      Cambridge, MA, 02139 USA
      \AND 
      \name Ali Syed  \email saabidi1@tamu.edu\\ \addr Department of Mathematics \\
      Texas A\&M University\\
      College Station, TX, 77840, USA  
      \AND 
      \name Nadav Dym \email nadavdym@technion.ac.il \\
       \addr Faculty of Mathematics\\
      Technion- Israel Institute of Technology\\
      Haifa, 3200003, Israel
      }

\editor{My editor}

\maketitle

\begin{abstract}
The universal approximation theorem establishes that neural networks can approximate any continuous function on a compact set. Later works in approximation theory provide quantitative approximation rates for ReLU networks on the class of $\alpha$-Hölder functions $f: [0,1]^N \to \mathbb{R}$. The goal of this paper is to provide similar quantitative approximation results in the context of group equivariant learning, where the learned $\alpha$-Hölder function is known to obey certain group symmetries. While there has been much interest in the literature in understanding the universal approximation properties of equivariant models, very few quantitative approximation results are known for equivariant models. 

In this paper, we bridge this gap by deriving quantitative approximation rates for several prominent group-equivariant and invariant architectures. The architectures that we consider include: the permutation-invariant Deep Sets architecture; the permutation-equivariant Sumformer and Transformer architectures; joint invariance to permutations and rigid motions using invariant networks based on frame averaging; and general bi-Lipschitz invariant models. Overall, we show that equally-sized ReLU MLPs and equivariant architectures are equally  expressive over equivariant functions. Thus, hard-coding equivariance does not result in a loss of expressivity or approximation power in these models.

\end{abstract}

\begin{keywords}
 Group  Equivariant Learning, Approximation Theory, Approximation Rates, Deep Sets, Equivariant Frames.
\end{keywords}

\tableofcontents

\section{Introduction}
One of the most well-known theoretical results in the study of neural networks is the universal approximation theorem \citep{cybenko1989approximation,carroll1989construction,gallant1992learning,irie1988capabilities,funahashi1989approximate,hornik1990universal}, which states that neural networks can approximate any continuous function,  to any desired accuracy, uniformly on any compact subset of $\mathbb{R}^d$. This result forms the cornerstone of the approximation theory of neural networks. However, the universal approximation theorem doesn't provide any quantitative information concerning the size of the required neural network, limiting its practical utility. To address this issue, later research focused on deriving approximation \emph{rates} for different types of neural networks over a variety of well-behaved function classes. For example,  quantitative approximation rates have been derived for both shallow neural networks \citep{barron2002universal,makovoz1996random,bach2017breaking,klusowski2018approximation,mhaskar1996neural,petrushev1998approximation,siegel2024sharp}, and deep ReLU neural networks \citep{yarotsky2017error,guhring2020error,petersen2018optimal,lu2021deep,shen2022optimal,siegel2023optimal}. We refer to \cite{devore2021neural} for an overview of the approximation theory of neural networks.

One of the main problems of interest in this direction is to determine optimal approximation rates for deep ReLU neural networks on classical smoothness spaces, such as Besov and Sobolev spaces. In the series of papers \cite{yarotsky2018optimal,shen2022optimal,lu2021deep,siegel2023optimal}, this problem has been given a complete solution --- at least for deep ReLU neural networks --- in terms of the total number of parameters. A special case of these results is the following. Suppose that a function $f:[0,1]^N\rightarrow \mathbb{R}$ is $\alpha$-H\"older. Then $f$ can be approximated to accuracy $\epsilon$ uniformly on $[0,1]^N$ by a deep ReLU neural network with $C\cdot (1/\epsilon)^{N/2\alpha}$ parameters, where $C=C(N,\alpha)$. Furthermore, this rate cannot be improved asymptotically as $\epsilon\rightarrow 0$. This quantifies exactly how large a deep ReLU network must be in order to achieve a certain approximation error, depending upon the smoothness of the target function.



The purpose of this paper is to address the problem of quantitative approximation rates for \textit{group equivariant} and \textit{invariant} neural networks. This is a  popular paradigm in the recent machine learning literature, where function spaces for machine learning tasks are designed to have the same group symmetries as the target function which is being approximated. For example, in image classification tasks, the goal is to map an image $x$ to a vector $f(x)$ representing the class the image belongs to (e.g., dog, cat,...). This task is invariant to translations, that is, $f(gx)=f(x)$ for every translation $g$. In other learning settings, the function $f$ to be learned is $G$ equivariant, meaning that applying a group transformation to the input is equivalent to applying the same transformation to the output: $f(gx)=g(f(x))$. 

Group equivariant learning refers to the practice of designing function classes, i.e., neural network architectures, which by design have the same invariant/equivariant structure as the functions 
they are designed to approximate. Examples include convolutional neural networks, which respect the translation symmetry of images \citep{alexnet,zhang2019making}, graph neural networks which respect the permutation symmetries of graphs \citep{xu2018how,ppgn}, neural networks for unordered point sets which respect their permutation symmetries \citep{zaheer2017deep,pointnet}, and their extension to point clouds and geometric graphs which additionally require symmetries to rigid motions \citep{satorras2021n,gasteiger2020directional}. More exotic examples include multi-permutation equivariance for problems where the input itself is a neural network \citep{navon2023equivariant}, Lorenz-equivariance for physical simulations \citep{bogatskiy2020lorentz}, $\mathrm{SL}(2,\mathbb{R})$ equivariance for polynomial optimization \citep{lawrence2024learning} and $\mathrm{SU}(2)$ equivariance for quantum neural networks \citep{su2}. By hard-coding symmetries into the network architecture, one can improve sample efficiency and out-of-distribution generalization \citep{elesedy2021provably, mei2021, Batzner2022-sr}. 

The development of equivariant models has been significantly influenced by theoretical research into their approximation capabilities. Research so far has focused on universality, that is,   
proving that a given equivariant model is universal, i.e., can approximate all equivariant functions. This guarantees that there exists a configuration of model parameters that is a good approximation of the target function (even if finding these parameters can be difficult due to limited optimization and generalization capabilities).  For example,  the pioneering works on multisets, DeepSets \citep{zaheer2017deep} and PointNet \citep{pointnet}, both proved their models' universality for multiset approximation. For graphs, the approximation abilities of \emph{message passing neural networks (MPNNs)} have inspired standard networks like GIN \citep{xu2018how}, while their limitations inspired many popular graph neural networks with improved approximation power \citep{goes,ppgn, NEURIPS2019_ea9268cb}. Universality of networks which are jointly equivariant to permutations and rigid motions was discussed in  \cite{dymuniversality,hordan2024complete,delle2024three,hordanweisfeiler,li2025on}. 

While universal approximation of equivariant networks is well-studied, very little is known regarding more quantitative estimates of their approximation rates. This is an interesting direction, because while both equivariant and non-equivariant models have been proven to be universal over broad model classes, their quantitative approximation rates have not been rigorously compared. In a sense, we ask: do equivariant models impart benefits via improved \emph{efficiency} of approximation, i.e. greater expressivity for a given number of parameters? Or, is it solely a matter of sample efficiency and optimization? In this work, we take the first steps towards addressing this gap.

\subsection{Main  Results}

We now give a somewhat informal description of our results. 

Our focus is on analyzing approximation rates of functions $f$ on a Euclidean vector space $V$ of dimension $N$, which are invariant, or equivariant, to the action of a group $G$. In addition to the assumption that $f$ is equivariant, we will assume that $f$ is $\alpha$-H\"older for some $\alpha \in (0,1]$. Under the H\"older assumption, it is known that ReLU networks with $\sim (1/\epsilon)^{N/2\alpha} $ neurons can approximate $f$ to $\epsilon$ accuracy. The first question we ask is whether the additional assumption of invariance/equivariance can reduce this approximation rate for ReLU networks (which are not equivariant by design). The intuition is that, since the value of $f$ on a point $x$ determines its value on all of its orbit, the exponential dependence on the dimension of $V$ should replaced with an exponential dependence on the dimension of the orbit space  $V/G $, which we denote by $N_G$ (the notion of dimension used here will be discussed in Section \ref{sec:prelim}). Namely, we expect approximation rates of $\sim (1/\epsilon)^{N_G/2\alpha} $.

The second question we consider is whether contemporary architectures, which are invariant or equivariant by design, can attain approximation rates of $\sim (1/\epsilon)^{N_G/2\alpha} $. A priori, a positive answer to the second question and a negative answer to the first question would show that the incorporation of equivariance into the model leads to more efficient approximation. Alternatively, a negative answer to the second question and a positive answer to the first question would imply that enforcing equivariance as a hard constraint leads to inferior approximation, which would be in line with other recent arguments in the literature that favor data augmentation or approximate equivariance \citep{wang2022approximately,tahmasebiachieving}. However, we ultimately find that in the scenarios we consider in this paper, the answers to both questions are positive. In other words, generic ReLU networks and specialized invariant networks achieve roughly the same expressivity (over invariant functions) as a function of network size. Therefore, differences in their practical behavior are not attributable to expressivity \emph{over invariant functions}, but instead to differences in optimization and sample complexity deriving from the distinct, restricted parameterization of invariant networks.  

\textbf{First setting: permutation invariant and equivariant models} Our analysis can be divided into three different settings. In the first setting, we consider  \textit{point sets}. Point sets (also referred to as point clouds), are $d\times n$ matrices which we think of as a collection of $n$ points in $\RR^d$,  
 $$X=(x_1,\ldots,x_n)\in \RR^{d\times n}.$$
 A permutation $\tau$ acts on  $X$ by reordering of the points, namely
	$$\tau\left(x_1,\ldots,x_n \right)=\left(x_{\tau(1)},
	\ldots, x_{\tau(n)} \right) .$$
Since the permutation group is finite, the dimensions of $V/G$ and $V$ are the same. Accordingly, our first question is irrelevant in this setting: ReLU networks will have the approximation rates we desire, as $N_G=N $. Accordingly, only the second question is relevant here. 

We focus first on the well known DeepSets permutation invariant model \citep{zaheer2017deep}, which enforces permutation invariance by applying the same ReLU neural network $\Phi $ independently to each point $x_i$, and then summing all points and applying an additional ``outer'' ReLU network $\rho$, resulting in the total formula
\begin{equation*}
	\rho\left(\sum_{i=1}^n\Phi(x_i)\right).
\end{equation*}  
Despite the simplicity of this model, and the fact that it cannot directly process interactions between points,  it was shown in \cite{zaheer2017deep} that DeepSets are universal, in the sense that they can approximate all permutation invariant functions. However, \cite{zweig2022exponential} argued that these models have inferior approximation abilities with respect to models which process pairwise interactions, like Set Transformers \citep{lee2019set}. Accordingly, it is not \emph{a priori} clear whether DeepSets can attain the same approximation rates as ReLU networks. Our first main result is that this is indeed the case: DeepSets can approximate permutation invariant functions which are $\alpha$-H\"older with the expected  approximation rate of $\sim (1/\epsilon)^{N_G/2\alpha} $, where in this setting $N_G=N=nd $. Moreover, these results also hold when the  domain $V$ is not a vector space but a subspace of $\RR^{d\times n}$ whose ``intrinsic'' dimension is $N$. 

As a second step, we also show how to leverage our results for DeepSets to prove similar approximation rates for two popular permutation \textbf{equivariant} models, Set Transformers \citep{lee2019set} and Sumformers \citep{Sumformer,segoluniversal}.

As a corollary to our result on DeepSets, we additionally obtain a result of independent interest, relating to the approximation of non-invariant functions by ReLU networks. This relates to the question of approximation rates over a compact domain $K\subset \RR^N $ whose ``intrinsic'' dimension $\dint $ is lower than the ambient dimension $N$. In several works \citep{shaham2018provable,chen2019efficient,labate2024low}, it was shown that if $K$ is a $\dint$ dimensional manifold residing in $\RR^N$,  then a $\alpha$-H\"older function $f:K \to \RR$ can be approximated by a ReLU network with $\sim \epsilon^{-\dint/\alpha}$ parameters, namely the network size depends  on the intrinsic dimension rather than the ambient dimension. A more general result which does not assume $K$ is a manifold, but only that it has  Minkowski dimension $\dint$, was proven by \cite{nakada2020adaptive}. In contrast,  we will show in this paper that we can  get an approximation rate of $\sim \epsilon^{-\frac{d_{\text{intrinsic}}}{2\alpha}}$, which is an improvement by a factor of two (inside the exponent) over the $\sim \epsilon^{-\frac{d_{\text{intrinsic}}}{\alpha}}$ factor obtained in previous work \citep{nakada2020adaptive}.

\textbf{Second setting: Models jointly invariant to permutations and rigid motions}
In this setting, we consider again point clouds $X\in \RR^{d\times n} $, where now we are interested in invariance not only to permutations, but also to the group of rigid motions $E(d) $. Namely, for a given orthogonal matrix $R$ and vector $t\in \RR^d$, we consider functions which are  invariant  to the action
	$$(R,t)\left(x_1,\ldots,x_n \right)=\left(Rx_1+t,\ldots,Rx_n+t \right) .$$
This type of symmetry typically occurs in dimension $d=3$, where the point set could be a representation of a 3D object in computer vision or a set of $n$ particles in particle dynamics.  There are many models in the recent machine learning literature taking these symmetries into account; prominent examples include EGNN \citep{satorras2021n}, Dimenet \citep{gasteiger2020directional}, Vector Neurons \citep{vn}, and MACE \citep{MACE}.

In this setting, the dimension of the quotient space $N_G=V/G $ will be lower than the dimension of $V$, namely 
$$N_G=nd-\mathrm{dim}(E(d)).$$
 Accordingly, we begin by addressing our first question, and show that  ReLU networks can attain the desired  approximation rates in this invariant setting. We then leverage these results, as well as our results from DeepSets, to show that invariant neural networks, based on applying rotation invariant frames to DeepSets, attain the desired approximation rates. Examples of such networks include networks based on alignment to the shapes' principal directions \citep{punyframe} and networks based on weighted frame averaging \citep{pozdnyakov2023smooth,dym2024equivariant}.  


\textbf{Third Setting: Bi-Lipschitz invariant models}
In this setting, we consider a more general setting of closed isometry groups acting on $\RR^D$, where the invariant models we consider obey the paradigm of bi-Lipschitz invariant models. We show that such models always attain the desired approximation rates. Examples of bi-Lipschitz models include max filter based models for finite groups \citep{cahill2024group,mixon2023max}, permutation invariant models based on multidimension sorting \citep{balan2022permutation,davidson2024} or the FSW embedding \citep{amir2024fsw,fswgnn}, and models for anti-symmetric functions based on bi-Lipschitz invariants for the alternating group \citep{Mizrachi}. 

Our main results are summarized in Table \ref{tab:results}. 
\begin{table}[ht]                                                                                                            
  \centering      
  \caption{ \it \small Summary of main results. Rates denote the asymptotic number of parameters
  needed for $\epsilon$-approximation as $\epsilon \rightarrow 0$ (up to logarithmic factors). Here, $n$ is the number of input elements of dimension $d$, and $\dint$ is the intrinsic dimension of the ``data domain'' quotiented by the group (see equation \eqref{eq:covering}).}
  \label{tab:results}
  \small
  \setlength{\extrarowheight}{4pt}
  \begin{tabular}{|>{\raggedright\arraybackslash}m{2.2cm}|>{\raggedright\arraybackslash}m{3.5cm}|>{\raggedright\arraybackslash}
  m{4cm}|>{\raggedright\arraybackslash}m{3.5cm}|}
  \hline
   & \textbf{Permutations} ($S_n$, \Cref{sec:permutations}) & \textbf{Perm.\ \& rigid motions} ($E(d)\times S_n$,
  \Cref{sec:rotation}) & \textbf{General isometries } ($G\leq O(N)$, \Cref{sec:bilip}) \\[6pt]
  \hline
  \textbf{Q1: ReLU}
  & $S_n$ finite, so standard rate of $(1/\epsilon)^{\frac{nd}{2\alpha}}$
  & \Cref{thm:reluO,cor:reluE,cor:perm_also}: $(1/\epsilon)^{\frac{nd-\dim(E(d))}{2\alpha}}$
  & \cref{cor:intrinsic}: $(1/\epsilon)^{\frac{\dint}{2\alpha}}$ 
  \\[6pt]
  \hline
  \textbf{Q2: Equivariant}
  & Theorems \ref{deep-sets-corollary}, \ref{thm:equi-fun}, \ref{cor:tr}: DeepSets, Sumformer, and Transformer, $(1/\epsilon)^{\frac{nd}{2\alpha}}$ 
  & \Cref{prop:reluE}: frame + ReLU, $(1/\epsilon)^{\frac{nd-\dim(E(d))}{2\alpha}}$ 
  & \Cref{thm:bi_lip}: bi-Lipschitz invariants, $(1/\epsilon)^{\frac{\dint}{2\alpha}}$ 
   \\[6pt]
  \hline
  \end{tabular}
  \end{table}

\textbf{Discussion and Future Work}
We believe our results in these three different settings are an important step towards a more quantitative approximation theory for group equivariant learning, as they establish for the first time approximation rates for several prominent models defined on equivariant data. Moreover, the ability of non-invariant ReLU networks to exploit group invariance and achieve expected approximation rates is consistent with empirical observations that, given sufficient data, non-invariant models can compete effectively with those that incorporate rotation invariance by design \citep{aug1,aug2}. This suggests that the benefits of hard-coded symmetries may manifest more significantly in sample efficiency and optimization rather than in asymptotic expressivity.

However, there remain several open directions for future work. 
Firstly, while we achieve approximation rates for rotation equivariant frames, we do not have approximation rates for many popular rotation equivariant models such as EGNN \citep{satorras2021n}, Dimenet \citep{gasteiger2020directional}, Vector Neurons \citep{vn}, and MACE \citep{MACE}. It would be interesting to establish approximation rates for these models, and see whether this analysis can yield provable gaps in the approximation rates between different equivariant methods. In such discussion one needs to take into account that many of these  models may not even be universal \citep{Pozdnyakov_2022,li2023is}, unless some additional restrictions to the domain are applied \citep{hordan2024complete,li2025on,sverdlov2025on}.  Simultaneously, it is possible that the study of asymptotic approximation rates, which ignores constants, is too coarse a measure to capture meaningful difference between models at realistic scales, and a still more refined analysis is necessary in order to, e.g, substantiate gaps in approximation quality between different equivariant models, or between equivariant and non-equivariant models. Examples of this type of analysis for the permutation group include \cite{zweig2022exponential,zweig2023antisymmetricneuralansatzseparation}.

\subsection{Related Work}
\paragraph{Approximation rates for permutation invariant models}
As discussed in the introduction, seminal works such as \citet{maron2019universality}, \citet{zaheer2017deep}, and \citet{Yarotsky2022InvariantMaps} establish the universality of various invariant architectures, but do not provide quantitative approximation rates. Before this work, the only quantitative approximation rates we were aware of for invariant functions pertained to the permutation group. \citet{han2022universal} provides explicit rates for permutation invariant functions with feature dimension scaling as $(1/ \epsilon)^{Nd}$, but the constructions are not necessarily implementable as neural architectures, requiring an arbitrary per-element featurization. We provide rates for architectures with the same asymptotics in $\epsilon$ that are implementable as practical neural networks for \ah functions. More recently, \citet{bachmayr2024polynomial} use invariant polynomials to approximate \emph{analytic} invariant functions, with an error exponential in the polynomial degree. Such results can be combined with \citet{siegel2023optimal} in order to obtain neural network rates. \citet{wagstaff2022universal} characterize necessary and sufficient embedding dimensions for universal set function approximation, providing lower bounds that illuminate the role of the latent dimension in DeepSets-type architectures. On the negative side, \citet{zweig2022exponential} demonstrate exponential separations between symmetric neural networks that process only element-wise features versus those with access to pairwise interactions, suggesting fundamental limitations of DeepSets relative to higher-order models. Our results show that despite this separation, DeepSets still achieves the optimal \emph{asymptotic} approximation rate over H\"older functions.

There is also a parallel line of work on the approximation of permutation-equivariant functions by the transformer architecture. \citet{yun2020are} established that transformers are universal over permutation-equivariant functions, for fixed width but arbitrary depth. \citet{TakeshitaImaizumi2025} derives approximation rates for transformers on symmetric polynomials, with more favorable dependence on $1/\epsilon$, though restricted to polynomial targets rather than the general H\"older function class considered here.

\paragraph{Intrinsic dimension and neural network approximation}
A related line of work studies how the approximation complexity of ReLU networks can be reduced when the domain has low intrinsic dimension. \citet{shaham2018provable} and \citet{chen2019efficient} showed that if the domain is a smooth manifold of dimension $\dint$ embedded in $\RR^N$, then the approximation rate depends on $\dint$ rather than $N$, and \citet{labate2024low} extended these results with refined generalization bounds. In a more general setting not requiring manifold structure, \citet{nakada2020adaptive} proved that approximation rates depend on the Minkowski dimension of the domain, obtaining a rate of $\sim\epsilon^{-\dint/\alpha}$ for $\alpha$-H\"older functions. We improve on this result by a factor of two in the exponent, obtaining $\sim\epsilon^{-\dint/(2\alpha)}$. These intrinsic dimension results are directly relevant to equivariant approximation: for a group $G$ acting freely on a space $V$, the orbit space $V/G$ has dimension $\dim(V) - \dim(G)$, and our framework exploits this reduction, connecting equivariant approximation to approximation over lower-dimensional quotient spaces.

\paragraph{Bi-Lipschitz invariant models and frame averaging}
Our results on bi-Lipschitz invariant models (Section~\ref{sec:bilip}) build upon a growing body of work constructing provably bi-Lipschitz invariant embeddings. For finite subgroups of the orthogonal group, \citet{cahill2024group} and \citet{mixon2023max} develop the theory of group-invariant max filtering, while \citet{max_mixon2025} and \citet{max_QADDURA} establish injectivity and local bi-Lipschitz properties of these maps on orbit spaces. General theoretical frameworks have been proposed by \citet{cahill2024towards} and \citet{balan2023g}. For the permutation group specifically, concrete bi-Lipschitz embeddings include those based on sorting \citep{balan2022permutation,davidson2024}, the Fourier Sliced-Wasserstein embedding \citep{amir2024fsw,fswgnn}, and bi-Lipschitz invariants for the alternating group \citep{Mizrachi}. In a complementary direction, frame averaging \citep{punyframe} provides a general recipe for achieving exact invariance by averaging over a discrete set of group elements, with \citet{dym2024equivariant} showing that continuous canonicalization is impossible for commonly used groups, motivating weighted frames \citep{pozdnyakov2023smooth}. Our work provides the first quantitative approximation rates for models based on these bi-Lipschitz and frame averaging approaches.

\paragraph{Generalization perspective}
While this paper focuses on approximation rates, a complementary line of work establishes that equivariance yields provable generalization benefits. \citet{elesedy2021provably} proved strict generalization improvements for equivariant models, \citet{mei2021} showed that invariance can save a factor of $d^\alpha$ in sample complexity, and \citet{sannai2021improved} derived tighter bounds via quotient feature spaces. \citet{trivedi} analyzed approximation-generalization trade-offs under approximate group equivariance. Our finding that equivariant and non-equivariant models achieve the same asymptotic approximation rates over invariant functions complements these results, suggesting that the practical advantages of equivariant architectures are primarily driven by sample efficiency rather than differences in expressivity.

\subsection{Paper Structure}
The structure of the remainder of the paper is as follows: In Section \ref{sec:prelim} we discuss some mathematical preliminaries necessary for a full statement and proof of our results. In Section \ref{sec:permutations} we discuss our results for the first setting, models invariant or equivariant to the action of permutations. In Section \ref{sec:rotation} we discuss our results for the second setting: invariant to both permutations and rigid motions, and in Section \ref{sec:bilip} we discuss the third setting: bi-Lipschitz models. 

\section{Preliminaries}\label{sec:prelim}
In this section we give several formal definitions necessary for a full formal statement and proof of our main results later on. 

\textbf{Notation}
We will use $\|\cdot \|$ without any subscript to denote the standard two-norm on vectors. $\|\cdot \|_F$ will denote the Frobenius matrix norm. 
We will use  $S_n$ to denote the group of permutations on $n$ elements. $O(d)$ denotes the group of orthogonal $d\times d$ matrices and $E(d)$ denotes the group generated by orthogonal transformations and translations by vectors in $\RR^d$.  The group of special orthogonal matrices (orthogonal matrices with determinant of 1) is denoted by $SO(d) $. 

\textbf{Invariance and Equivariance}
Let $V$ and $W$ be two sets, and $G$ a group acting on these sets. We say that $f$ is \emph{equivariant} if $f(gv)=gf(v) $ for all $v\in V$ and $g\in G$. We say that $f$ is invariant in the case where the action of $G$ on $W$ is trivial, namely $g.w=w$ for all $g\in G,w\in W$, so that we obtain $f(gv)=f(v)$.

\textbf{Regularity}
We suppose that in addition to being invariant or equivariant, the function $f$ satisfies a certain \emph{smoothness condition}. In this work, we will measure smoothness in the H\"older spaces. Given an $0 < \alpha \leq 1$, the $\alpha$-H\"older semi-norm of a function $f$ is defined by
\begin{equation}
	|f|_{C^\alpha} := \sup_{x\neq y} \frac{|f(x) - f(y)|}{\|x - y\|^\alpha}.
\end{equation}
We will suppose that the target function $f$ satisfies $|f|_{C^\alpha} \leq C$ for some fixed constant $C$, and is invariant or equivariant. We can more generally consider smoothness through the modulus of continuity, defined by
\begin{equation}\label{modulus-of-continuity-intro}
	\omega(f,\epsilon) := \sup_{\substack{x,y\in \Omega\\|x - y| \leq \epsilon}} |f(x) - f(y)|.
\end{equation}
H\"older continuity is then just equivalent to $\omega(f,\epsilon) \leq C\epsilon^{\alpha}$. We remark that these notions depend upon a choice of norm $\|\cdot\|$ on $V$, which we by default typically taken to the Euclidean norm. However, in some of our results  it will be more convenient to use the $\ell_\infty$-norm on $\RR^N$ instead. In any case, the choice of norm will only affect the final results by changing the constants involved, due to the equivalence of norms in a finite dimensional space.  
 
Let us remark that using the modulus of continuity only allows our analysis to consider smoothness of order up to a single derivative. In order to handle higher order smoothness, we would also need to consider higher order moduli of smoothness (see for instance Chapter 2 in \cite{devore1993constructive}). However, this will introduce significant technical complications in our arguments, and for the sake of simplicity we leave this for future work. 

\textbf{Quotient space} For a group $G$ acting on a set $V$, we denote the orbit of an element $v\in V$ by $[v]$. The set of orbits is denoted by $V/G$ and we refer to it as the quotient space. The quotient map $q:V\to V/G$ is the mapping 
$$q(v)=[v]. $$

When $V=\RR^N $ and $G$ is a closed subgroup of the group of isometries of $V$, namely $E(N)$, then we can define a metric on the quotient space $V/G$ via 
\begin{equation}\label{eq:dG}
d_G([x],[y])=\min_{g\in G} \|gx-y\| .
\end{equation}
Since the elements of $G$ are isometries, this definition does not depend on the choice of representatives $x\in X,y\in Y$ and is a well defined metric.  

\textbf{Covering Number and Dimension}
Our results will search for an approximation of a Holder function on a compact set $K$ using neural networks. The size of the networks will be  strongly tied to the covering number of $K$. We recall that for a metric space $(X,d)$, and a number $r>0$, we define $N(X,r)$ to be the smallest number of balls of radius $r$ in $X$ which form a cover of $X$.  For example, when $K=[0,1]^N$ the covering number of $K$ satisfies, for an appropriate positive constant $C$, the inequality  
\begin{equation}\label{eq:covering}
N_r(K)\leq Cr^{-N}, \quad \forall r>0 
\end{equation}
As we are interested in $G$ invariant and equivariant functions, we will be able to think of our functions as functions defined on the quotient space $K/G $, and we will be interested in the covering number of $K/G$ with respect to the quotient metric $d_G$. Generally speaking, we will expect to get an equation of the form \eqref{eq:covering}, where $N $ is replace with $N-\dim(G) $. The intuition for this is as follows: assume that  $G$ is a Lie Group, with a free and proper action on $\RR^N$. Recall that a proper action means that the map $\Phi(g,x)=(gx,x) $ from $G\times \RR^N$ to $\RR^N\times \RR^N $ is proper, which in turn means that for every compact set $K\subseteq \RR^N\times \RR^N$, the preimage $\Phi^{-1}(K) $ will be compact. A free action means that if $gx=x$ for some $x\in \RR^N$, then $g$ is the identity element in $G$. 
In this case,  the quotient space is a manifold (see e.g., \cite{duistermaat2012lie}), with a dimension of 
\begin{equation}\label{eq:dim_lie}
\dim(\RR^N/G)=N-\dim(G),
\end{equation}
Accordingly, if we wish to approximate $f$ on $K/G$, a compact subset of $\RR^N/G $, then  we will indeed be able to get a bound as in \eqref{eq:covering}. 


In the point cloud setting we are interested in, the groups are Lie groups, and their action is proper, but it is typically not free. Thus, the dimension equality in \eqref{eq:dim_lie} may not hold. For  example, in the case $n=1,d\geq 3$, we have that $\RR^d/O(d) $ can be identified with $[0,\infty)$ (since the orbit of a vector $x$ under $O(d)$ is the sphere of radius $\|x\|$), and so has dimension
 
$$\dim(\RR^d/O(d))=1 > d- \frac{d(d-1)}{2}=\dim(\RR^d)-\dim(O(d)), \quad \forall d\geq 3 $$
Nonetheless, in the cases we focus on, the covering number of (compact subsets of) $\RR^{d\times n}/G $ does scale like $dn-\dim(G) $ (despite the fact that $\RR^{d\times n}/G $ is not a manifold)

\begin{prop}\label{prop:dim}[Proof in appendix]
Let  $d,n$ be natural numbers and let $G$ be a group acting on $\RR^{d\times n}$, and assume that one of the  following holds:
\begin{enumerate}
\item $G$ is the permutation group, $G=S_n$.
\item  $G=E(d)\times S_n $, and $d+1\leq n$.
\end{enumerate}
Then for any $G$ invariant set $K\subseteq \RR^{d\times n}$ for which $K/G$ is compact, 
$$N_r(K/G)\leq Cr^{-(nd-\dim(G))}, \quad \forall r>0.$$
\end{prop}
In the first case where $G=S_n$, the group is finite and so has dimension zero. In the second case $G$ does have positive dimension. 
This proposition gives the intuition as to why the approximation rates we derive for permutation groups in Section \ref{sec:permutations} and  for $E(d)\times S_n$ in Section \ref{sec:rotation} are the ``correct'' rates. We stress however that these results do not depend on this proposition. The assumption that $d+1\leq n$
is not very restrictive as in applications typically $d=3, n\gg 4$.


We also note that an assumption like \eqref{eq:covering} is almost equivalent to asserting that the upper-Minkowsi dimension of $K$ is $N$. Indeed, the  upper-Minkowski dimension of a set $K$ is defined as 
$$\dimM\left(K \right)=\limsup_{r\rightarrow 0} \frac{\log(N(K,r))}{\log(1/r)} .$$
If \eqref{eq:covering} holds, it follows that $\dimM\left(K \right)\leq N $. Conversely, if 
$\dimM\left(K\right)=N $, this implies a slightly weaker statement: namely, that for all $\eta>0$, there exists a constant $C=C_\eta$ such that,
\begin{equation}\label{eq:minkowski}
N(r)\leq C(1/r)^{d+\eta} \quad \forall r\in (0,1). 
\end{equation}


\section{Approximation Rates for Permutation-Equivariant Models}\label{sec:permutations}

In this section, we consider approximation rates for permutation-invariant and equivariant functions. Technically, the most significant result in this section is the approximation rate we derive for the permutation-invariant DeepSets model. These results are later used to derive approximation rates for two popular permutation equivariant architectures: Sumformer and Set Transformer. We will then obtain as a corollary improved approximation rates for non-invariant $\alpha$-H\"{o}lder functions defined on domains with `low' Minkowski dimension. 

\subsection{Approximation Rates for the Deep Sets  Permutation Invariant Model}\label{subsec:DeepSets}
We begin by deriving quantitative approximation rates for the Deep Sets architecture introduced in \cite{zaheer2017deep}. In fact, it will be convenient to consider the following generalization of Deep Sets, which will prove useful when analyzing permutation equivariant architectures. Specifically, consider the index set $[n] := \{1,...,n\}$ and let $P = \{P_i\}_{i=1}^k$ be a partition of $[n]$. Consider the group
\begin{equation}
    S_P := \{\sigma\in S_n:~\sigma(P_i) = P_i~\text{for $i=1,...,k$}\} \subset S_n,
\end{equation}
which is isomorphic to a product of symmetric groups on $|P_i|$ elements, i.e.
\begin{equation}
    S_P \eqsim \prod_{i=1}^k S_{|P_i|}.
\end{equation}
For a given partition $P$, we consider functions $f:\mathbb{R}^{d\times n}\rightarrow \mathbb{R}$ which are invariant to permutations in  $S_P$. The case where  $P = \{[n]\}$,  corresponds to full permutation invariance.

We consider  $S_P$ invariant neural networks which are of the form
\begin{equation}\label{generalized-deep-sets}
    \rho\left(\sum_{i\in P_1}\Phi(x_i),...,\sum_{i\in P_k}\Phi(x_i)\right),
\end{equation}
where $\rho:\mathbb{R}^{N\times k}\rightarrow \mathbb{R}$ and $\Phi:\mathbb{R}^d\rightarrow \mathbb{R}^N$ are parameterized by ReLU neural networks. When $P = \{[n]\}$, this reduces to the DeepSets architecture introduced in \cite{zaheer2017deep}:
\begin{equation}
    \rho\left(\sum_{i=1}^n\Phi(x_i)\right).
\end{equation}
We remark that in \eqref{generalized-deep-sets} we could have also use a different network $\Phi$ for each block of the partition $P$, but we will actually show that this is not necessary to obtain our approximation rates (and universality). We also remark that the groups $S_P$ are examples of groups for which first order networks are universal (see for instance \cite{maron2019universality}).

We wish to obtain approximation rates using the architecture \eqref{generalized-deep-sets} where the networks $\rho$ and $\Phi$ are parameterized by deep ReLU neural networks. In order to obtain quantitative rates, we must make quantitative assumptions on the target function $f$. As remarked in the introduction (see \eqref{modulus-of-continuity-intro}), we will measure the complexity of $f$ using the modulus of continuity with respect to the $\ell_\infty$-norm on the domain $\Omega$:
\begin{equation}
    \omega(f,\epsilon) := \sup_{\substack{x,y\in \Omega\\\|x - y\|_{\infty} \leq \epsilon}} |f(x) - f(y)|.
\end{equation}
Our main result is the following upper bound on the achievable approximation rates.

\begin{theorem}\label{main-deep-sets-approximation-theorem}
	Let $P$ be a partition of $[n]$, let $\Omega = [0,1]^d$, and suppose that $f:\Omega^n\rightarrow \mathbb{R}$ is an $S_P$ invariant function. Let $K\subset \Omega^n$ be an $S_P$-invariant subset, and denote by $N_r(K)$ the smallest number of balls (in the $\ell_\infty$-norm) of radius $r$ required to cover $K$ modulo the action of $S_P$.
	
	Let $m \geq 1$ be an integer and set $N = 2dn+1$. Then there exist ReLU neural networks $\Phi:\mathbb{R}^d\rightarrow \mathbb{R}^N$ and $\rho:\mathbb{R}^{N\times k}\rightarrow \mathbb{R}$ satisfying that 	for every $X = (x_1,...,x_n)\in K$,
		\begin{equation}\label{deep-sets-equation-bound}
		\left|f(X) - \rho\left(\sum_{i\in P_1}\Phi(x_i),...,\sum_{i\in P_k}\Phi(x_i)\right)\right| \leq 2\omega\left(f,\frac{1}{2m}\right),
	\end{equation}
	where
	\begin{itemize}
		\item The network $\Phi$ has three hidden layers of width $W = 2d(2dn+1)(m+1)$, followed by one hidden layer of width $(m+1)^d$, and an output layer of width $2nd+1$.
		\item The network $\rho$ consists of $C(m+1)^d$ layers of width $C(n+1)$ followed by $L$ layers of width $W = 12(2dn+1)$, where
		$$r = \frac{1}{2m(2nd+1)}, \quad L \leq C\max\{\sqrt{N_r(K)\log[n(2m+1)N_r(K)]},\log(2nd+1)^2\},$$
	and $C$ is an absolute constant.
	\end{itemize}
\end{theorem}

%
Before giving the proof of Theorem \ref{main-deep-sets-approximation-theorem}, let us note the following corollary. Here we consider the case of Deep Sets approximating permutation invariant functions, i.e., the partition into a single set containing all elements, $P = \{[n]\} $. The assumptions on $f$ are $\alpha$-H\"older continuity, i.e., that $\omega(f,t) \leq t^\alpha$ for some $\alpha\in (0,1]$, and the rates are given in terms of the intrinsic dimension of the set $K$.
\begin{cor}\label{deep-sets-corollary}
    Let $n,d,\dint$ be fixed natural numbers with $\dint\geq 2d+2 $, denote $\Omega = [0,1]^d$, and let $C_1$ and $0 < \alpha \leq 1$, and $0 < \epsilon \leq 1$ be real numbers. Let $K\subseteq \Omega^n$ be a permutation invariant set with 
    \begin{equation}\label{eq:coveringK}
    N_r(K/S_n)\leq C_1r^{-\dint}, \quad \forall r>0 \end{equation}
    For $N = 2dn+1$ there exist ReLU neural networks $\Phi:\mathbb{R}^d\rightarrow \mathbb{R}^N$ and $\rho:\mathbb{R}^N\rightarrow \mathbb{R}$ with a total number of parameters $$P\leq C\epsilon^{-\dint/(2\alpha)}(1+|\log\epsilon|),$$ such that for any permutation invariant function $f:\Omega^n\rightarrow \mathbb{R}$ satisfying $\omega(f,t) \leq t^\alpha$ there is a choice of parameters such that
    \begin{equation}
        \left|f(X) - \rho\left(\sum_{i=1}^n \Phi(x_i)\right)\right| \leq \epsilon
    \end{equation}
    for every $X = (x_1,...,x_n)\in K$. Here $C$ is a constant depending only on $d,n,\dint,C_1$ and $\alpha$, but not on $\epsilon$.
\end{cor}
This result follows from Theorem \ref{main-deep-sets-approximation-theorem} by setting $P = [n]$, $m = C\epsilon^{-1/\alpha}$, counting the number of parameters in the networks $\Phi$ and $\rho$, and noting that $(m+1)^{\dint/2}$ dominates both $(m+1)^2$ and $(m+1)^{d+1}$ as $m\rightarrow \infty$ since $\dint \geq 2d+2$. 

In particular, in the case that $K=\Omega^n $ we have that \eqref{eq:coveringK} holds with $\dint=nd $. Under the mild assumption that $n\geq 4$ we have that $nd\geq 2d+2 $ for all $d$, and so 
Corollary \ref{deep-sets-corollary} implies that on the class of $\alpha$-H\"older continuous functions in dimension $D = dn$, Deep Sets matches up to a logarithmic factor the optimal approximation rates for general ReLU MLPs (see for instance \cite{siegel2023optimal,yarotsky2018optimal,shen2022optimal}). This indicates that there is no loss of expressivity when using the Deep Sets architecture to enforce permutation invariance.
\newline
\begin{proof}\textbf{of Theorem \ref{main-deep-sets-approximation-theorem}}
First, let us note that we may assume without loss of generality that the partition $P$ consists of consecutive pieces, i.e., if $n_j := |P_j|$, then $P_j = \{n_1 + n_2 + ... + n_{j-1} + 1,...,n_1 + n_2 + ... + n_{j}\}$. This serves to simplify the presentation of the argument.

Before giving the proof of Theorem \ref{main-deep-sets-approximation-theorem}, let us describe our approach by first constructing \textit{discontinuous} functions $\widetilde{\Phi}:\mathbb{R}^d\rightarrow \mathbb{R}$ and $\widetilde{\rho}:\mathbb{R}^k\rightarrow \mathbb{R}$ satisfying \eqref{deep-sets-equation-bound}. Of course, this is trivial to do, but the specific $\widetilde{\Phi}$ and $\widetilde{\rho}$ that we construct will later be modified to obtain functions $\Phi$ and $\rho$ which can be explicitly represented using deep ReLU networks and which will satisfy \eqref{deep-sets-equation-bound}.

We begin by partitioning the unit cube $\Omega := [0,1]^d$ into $m^d$ sub-cubes of diameter at most $1/m$. Specifically, we define
\begin{equation}
    \Omega_{\textbf{i}} := \frac{1}{m}\prod_{j=1}^d [\textbf{i}_j, \textbf{i}_j+1),~\text{for $\textbf{i}\in \{0,...,m-1\}^d$,}
\end{equation}
where if $\textbf{i}_j = m-1$, then the half-open interval $[\textbf{i}_j, \textbf{i}_j+1)$ is replaced by the closed interval $[\textbf{i}_j, \textbf{i}_j+1]$ in the above product. The $m^d$ sub-cubes $\Omega_\textbf{i}$ are disjoint and cover the cube $\Omega$. We also denote by $z_\textbf{i}\in \Omega$ the midpoint of the cube $\Omega_\textbf{i}$ given by
\begin{equation}
    (z_\textbf{i})_j = \frac{\textbf{i}_j + \frac{1}{2}}{m}~\text{for $j=1,...,d$},
\end{equation}
and observe that for any $x\in \Omega_\textbf{i}$, we have $\|x - z_\textbf{i}\|_{\infty} \leq 1/(2m)$.

Set $M := m^d$. We first construct a map $\overline{\Phi}:\mathbb{R}^d\rightarrow \mathbb{R}^M$ as follows. Let $\chi_{\textbf{i}}$ denote the characteristic function of the sub-cube $\Omega_\textbf{i}$, i.e., we set
\begin{equation}
    \chi_{\textbf{i}}(x) = \begin{cases}
        1 & x\in \Omega_\textbf{i}\\
        0 & x\notin \Omega_\textbf{i}.
    \end{cases}
\end{equation}
We use the set $\{0,...,m-1\}^d$ to index the components of $\mathbb{R}^M$ and define the map $\overline{\Phi}$ as
\begin{equation}
    \overline{\Phi}(x)_{\textbf{i}} := \chi_{\textbf{i}}(x).
\end{equation}
In other words, the $\textbf{i}$-th component of $\overline{\Phi}$ is the characteristic function $\chi_\textbf{i}$. 

Observe that if $X = (x_i)_{i=1}^n\in \Omega^n$, then the inner sums in \eqref{deep-sets-equation-bound} using the map $\overline{\Phi}$ are given by
\begin{equation}
    \sum_{i\in P_j}\overline{\Phi}(x_i)\in \mathbb{Z}^{M}~\text{with}~\left[\sum_{i\in P_j} \overline{\Phi}(x_i)\right]_{\textbf{i}} = |\{i\in P_j:x_i\in \Omega_{\textbf{i}}\}|,
\end{equation}
i.e., the $\textbf{i}$-th component of this sum is equal to the number of coordinates $x_i$  in the partition $P_j$ which lie in $\Omega_\textbf{i}$. 

For an index $j$ we let $n_j := |P_j|$. Let $\tau\in \mathbb{Z}^M$ be such that $\tau_{\textbf{i}} \geq 0$ and $\sum_{\textbf{i}} \tau_\textbf{i} = n_j$. We let $Y(\tau)$ be a point $Y\in \Omega^{n_j}$ such that $\tau_\textbf{i}$ components of $Y$ are equal to $z_{\textbf{i}}$ for each $\textbf{i}$. Such a $Y$ is unique up to permutation and we choose $Y(\tau)$ to be any such point. It is easy to see that for any $X = (x_i)_{i=1}^{n_j}\in \Omega^{n_j}$ there exists a permutation $\sigma\in S_{n_j}$ such that
\begin{equation}\label{reordered-sigma-equation}
    \left\|\sigma(X) - Y\left(\sum_{i=1}^{n_j} \overline{\Phi}(x_i)\right)\right\|_{\infty} \leq \frac{1}{2m}.
\end{equation}
Indeed, by reordering the points of $X$ we can match each $x_i\in \Omega_\textbf{i}$ with a corresponding $z_\textbf{i}$ component of $Y$, since the number of $z_\textbf{i}$ components of $Y$ is equal to the number of $x_i$'s in $\Omega_\textbf{i}$.

We now define the map $\overline{\rho}:\mathbb{R}^{M\times k}\rightarrow \mathbb{R}$ such that
\begin{equation}
    \overline{\rho}(\tau_1,...,\tau_k) = f(Y(\tau_1),...,Y(\tau_k))
\end{equation}
for each $(\tau_1,...,\tau_k)\in \mathbb{Z}^{M\times k}$ with $(\tau_j)_{\textbf{i}} \geq 0$ and $\sum_{\textbf{i}} (\tau_j)_\textbf{i} = n_j$. It then follows from the $S_P$ invariance of $f$ that (here $\sigma\in S_P$ is obtained by applying each of the permutations in \eqref{reordered-sigma-equation} to the corresponding block $P_j$.)
\begin{equation}
\begin{split}
    \left|f(X) - \overline{\rho}\left(\sum_{i\in P_1}\overline{\Phi}(x_i),...,\sum_{i\in P_k}\overline{\Phi}(x_i)\right)\right| &= \left|f(\sigma(X)) - f\left(Y\left(\sum_{i\in P_1} \overline{\Phi}(x_i)\right),...,Y\left(\sum_{i\in P_k} \overline{\Phi}(x_i)\right)\right)\right|\\
    &\leq \omega\left(f,\frac{1}{2m}\right).
\end{split}
\end{equation}
We remark that our construction so far is similar to the construction given in Section 3 of \cite{han2022universal}. However, our construction is significantly simpler and results in a much smaller number of latent variables $M$.

The next step is to construct $\widetilde{\Phi}$ and $\widetilde{\rho}$ by reducing the embedding dimension from $M$ to one. We will do this by encoding the number of points in each cell using the digits of a large integer. Specifically, for $\textbf{i}\in \{0,...,m-1\}^d$, we define
\begin{equation}
    |\textbf{i}| := \sum_{j=1}^d \textbf{i}_jm^j.
\end{equation}
The map $\widetilde{\Phi}$ is given by
\begin{equation}\label{digit-encoding-equation}
    \widetilde{\Phi}(x) = \sum_{\textbf{i}\in \{0,...,m-1\}^d} \chi_{\textbf{i}}(x)(n+1)^{|\textbf{i}|}.
\end{equation}
In words, $\widetilde{\Phi}(x)$ is the integer whose expansion in base $(n+1)$ contains a one in the position corresponding to the sub-cube $\Omega_\textbf{i}$ which contains $x$ and a zero in every other position. We observe that the inner sums in \eqref{deep-sets-equation-bound} using $\widetilde{\Phi}$ are
\begin{equation}\label{sum-of-Phi-tilde}
    \sum_{i\in P_j}\widetilde\Phi(x_i) = \sum_{\textbf{i}\in \{0,...,m-1\}^d} \left[\sum_{i\in P_j}\chi_{\textbf{i}}(x_i)\right](n+1)^{|\textbf{i}|}.
\end{equation}
Since the number of terms in the inner sum is at most $|P_j| \leq n$, it follows that
\begin{equation}
    \left[\sum_{i\in P_j}\chi_{\textbf{i}}(x_i)\right] \leq n
\end{equation}
for all $j$ and $X\in \Omega^n$. Hence, the sum in \eqref{sum-of-Phi-tilde} is an integer whose digits in base $n$ count the number of points $x_i$ which lie in the the sub-cube $\Omega_\textbf{i}$. We can now define $\widetilde{\rho}:\mathbb{R}^k\rightarrow \mathbb{R}$ so that
\begin{equation}
    \widetilde{\rho}\left(\sum_{\textbf{i}}(\tau_1)_\textbf{i}(n+1)^{|\textbf{i}|},...,\sum_{\textbf{i}}(\tau_k)_\textbf{i}(n+1)^{|\textbf{i}|}\right) = f(Y(\tau_1),...,Y(\tau_k))
\end{equation}
for each $(\tau_1,...,\tau_k)\in \mathbb{Z}^{M\times k}$ with $(\tau_j)_{\textbf{i}} \geq 0$ and $\sum_{\textbf{i}} (\tau_j)_\textbf{i} = n_j$. It then follows that
\begin{equation}
    \widetilde{\rho}\left(\sum_{i\in P_1}\widetilde{\Phi}(x_i),...,\sum_{i\in P_k}\widetilde{\Phi}(x_i)\right) = \overline{\rho}\left(\sum_{i\in P_1}\overline{\Phi}(x_i),...,\sum_{i\in P_k}\overline{\Phi}(x_i)\right)
\end{equation}
and we have reduced the latent dimension to one.

The problem with the $\widetilde{\Phi}$ and $\widetilde{\rho}$ which we have constructed is that $\widetilde{\Phi}$ is discontinuous ($\widetilde{\rho}$ could actually be taken continuous), and so it cannot even be approximated by ReLU neural networks. In order to prove Theorem \ref{main-deep-sets-approximation-theorem} we must construct maps $\Phi$ and $\rho$ which are represented by ReLU networks and which satisfy \eqref{deep-sets-equation-bound}, which is significantly more challenging. Our idea for achieving this is inspired by the proof of the Kolmorogov-Arnold superposition theorem (see for instance \cite{kolmogorov1957representations,kahane1975theoreme}).

We first introduce the following partitions of unity on $\mathbb{R}^d$. Let $0 < \delta < 1/2$ (which will be specified later) and consider the functions
\begin{equation}\label{definition-phi-i1}
        \phi_k(x) := \max\left\{0,\min\left\{\frac{1}{\delta}\left(\frac{1+\delta}{2} - \left|x - k\right|\right), 1\right\}\right\}
    \end{equation}
    for $k\in \mathbb{Z}$. It is a simple matter to verify that the collection of functions $\{\phi_k\}_{k\in\mathbb{Z}}$ is partition of unity on $\mathbb{R}$, i.e., that $\phi_k \geq 0$ and
    $$
        \sum_{k=-\infty}^{\infty}\phi_k(x) = 1
    $$
    for all $x\in \mathbb{R}$. Moreover, each $\phi_k$ vanishes outside the interval $$I_k := \left(k-\frac{1+\delta}{2}, k+\frac{1+\delta}{2}\right),$$
    and $\phi_k(x) = 1$ on the interval
    \begin{equation}
        \tilde{I}_k := \left[k-\frac{1-\delta}{2}, k+\frac{1-\delta}{2}\right].
    \end{equation}
    It is also clear from \eqref{definition-phi-i1} that each of the functions $\phi_k$ can be expressed using a ReLU neural network with width $W = 2$ and depth $L = 3$. 

    From the partition of unity $\{\phi_k\}_{k\in\mathbb{Z}}$ we can obtain a partition of unity on $\mathbb{R}^d$ via a tensor product construction. Specifically, for an index $\textbf{i}\in \mathbb{Z}^d$ we set
    \begin{equation}
        \phi_{\textbf{i}}(x) := \prod_{j=1}^d\phi_{\textbf{i}_j}(x_j).
    \end{equation}
    Then the functions $\phi_{\textbf{i}}$ for $\textbf{i}\in \mathbb{Z}^d$ form a partition of unity on $\mathbb{R}^d$. Given the integer $m \geq 1$, we would like to consider the scaled partition of unity
    \begin{equation}
        \phi_{\textbf{i},m}(x) := \phi_{\textbf{i}}(mx)
    \end{equation}
    restricted to the unit cube $\Omega$. Observe that exactly $(m+1)^d$ of the functions $\phi_{\textbf{i},m}$ are non-zero on $\Omega$ (namely those for which $\textbf{i}\in \{0,...,m\}^d$). We would like these $M := (m+1)^d$ functions to be continuous substitutes for the characteristic functions $\chi_{\textbf{i}}$ in our discontinuous construction (note we have slightly redefined the variable $M$). However, the overlap in their supports means that the resulting map $\Phi$ would no longer map into $\mathbb{Z}^M$, but rather into a much more complicated $d$-dimensional cell complex in $\mathbb{R}^M$. We would have to control the map $\rho$ on this set which is difficult to do using ReLU networks and introduces significant technical complications.

    Our method for overcoming this is to instead use multiple shifted partitions of unity, in a manner similar to the proof of the Kolmogorov-Arnold superposition theorem. 
    
    Specifically, we set $\delta = \frac{1}{2dn+1}$ and consider the following collection of $2dn+1$ shifted partitions of unity. For $q=0,...,2dn$ and $\textbf{i}\in \mathbb{Z}^d$ we set (here $\textbf{1}$ denotes the vector of ones)
    \begin{equation}\label{shifted-partition-definition-equation}
        \phi^q_{\textbf{i}}(x) := \phi_{\textbf{i}}(x - q\delta \textbf{1}) = \prod_{j=1}^d\phi_{\textbf{i}_j}(x_j - q\delta).
    \end{equation}
    We then write $\Pi_q := \{\phi^q_{\textbf{i}},~\textbf{i}\in \mathbb{Z}^d\}$ for the $q$-th shifted partition of unity. 
    
    For a given point $x\in \mathbb{R}^d$, we say that the partition $\Pi_q$ is \textit{bad} for $x$ if there exists an $\textbf{i}\in \mathbb{Z}^d$ such that $0 < \phi^q_{\textbf{i}}(x) < 1$. Otherwise, we say that the partition is \textit{good} for $x$. For a point $X = (x_1,...,x_n)\in \mathbb{R}^{dn}$ we say that the partition $\Pi_q$ is \textit{bad} for $X$ if it is bad for any coordinate $x_i$, and \textit{good} if it is good for all coordinates $x_i$.
    
    The key observation, inspired by the proof of the Kolmogorov-Arnold theorem, is that for each point $x\in \mathbb{R}^d$ the number of partitions $\Pi_q$ which are bad for $x$ is at most $d$. This follows since if $\Pi_q$ is bad for $x$ then there must be an index $j \in \{1,...,d\}$ and an index $k\in \mathbb{Z}$ such that $x_j - q\delta\in I_k \backslash \tilde{I}_k$. Indeed, otherwise every factor on the right hand side in \eqref{shifted-partition-definition-equation} would be either $0$ or $1$, so that the product would be either $0$ or $1$. We observe that 
    \begin{equation}
        \bigcup_{k\in \mathbb{Z}} I_k \backslash \tilde{I}_k = \bigcup_{k\in \mathbb{Z}} \left(k + \frac{1-\delta}{2}, k + \frac{1+\delta}{2}\right),
    \end{equation}
    and since $\delta = \frac{1}{2dn+1}$ it follows that the sets
    \begin{equation}
        B_q := q\delta + \bigcup_{k\in \mathbb{Z}} I_k \backslash \tilde{I}_k
    \end{equation}
    are disjoint for $q = 0,...,2dn$. This means that for each $j$ there is at most one value of $q$ such that $x_j\in B_q$, i.e., such that $x_j - q\delta\in I_k \backslash \tilde{I}_k$ for some $k\in \mathbb{Z}$. Hence, there is at most one bad value of $q$ for each index $j$, and so at most $d$ in total. We deduce from this that for any point $X = (x_1,...,x_n)\in \mathbb{R}^{dn}$, the number of bad partitions is at most $dn$, since each coordinate $x_i$ has at most $d$ bad partitions. Since the number of partitions is $2dn+1$, this implies that for each $X\in \mathbb{R}^{dn}$ only a minority of the partitions $\Pi_q$ are bad.

    Next, we modify our partition of unity. We define
    \begin{equation}\label{definition-of-tilde-phi}
        \widetilde{\phi^q_{\textbf{i}}}(x) := \max\left\{0,\sum_{j=1}^d \phi_{\textbf{i}_j}(x_j - q\delta) - d + 1\right\}.
    \end{equation}
    Although these functions no longer form a partition of unity, we still have
    \begin{equation}
        \widetilde{\phi^q_{\textbf{i}}}(x) = \phi^q_{\textbf{i}}(x)
    \end{equation}
    whenever the partition $\Pi_q$ is good for $x$. Indeed, if $\Pi_q$ is good for $x$, then either all terms in the product in \eqref{shifted-partition-definition-equation} are $1$ or one of them is $0$. In the first case, we have
    \begin{equation}
        \sum_{j=1}^d \phi_{\textbf{i}_j}(x_j - q\delta) - d + 1 = 1
    \end{equation}
    so that $\widetilde{\phi^q_{\textbf{i}}}(x) = 1$, while in the second case we have
    \begin{equation}
        \sum_{j=1}^d \phi_{\textbf{i}_j}(x_j - q\delta) - d + 1 \leq 0
    \end{equation}
    so that $\widetilde{\phi^q_{\textbf{i}}}(x) = 0$. Additionally, we see from this argument that if $\phi_{\textbf{i}}^q(x) = 0$, then also $\widetilde{\phi^q_{\textbf{i}}}(x) = 0$.

    Using these observations, we construct the network $\Phi$ as follows. For $\textbf{i}\in \mathbb{Z}^d$ we denote by
    \begin{equation}
        \phi^q_{\textbf{i},m}(x) := \widetilde{\phi^q_{\textbf{i}}}(mx).
    \end{equation}
    Next, we set
    \begin{equation}
        I_{q,m} = \mathbb{Z}\cap \left(-\frac{1+\delta}{2} - q\delta,~m + \frac{1+\delta}{2}-q\delta\right),
    \end{equation}
    and write
    \begin{equation}
        \mathcal{I}_{q,m} = \prod_{j=1}^d I_{q,m} \subset \mathbb{Z}^d.
    \end{equation}
    Observe that if $k\notin I_{q,m}$, then $(I_k + q\delta) \cap [0,m] = \emptyset$, and so $\phi_k(m(x - q\delta)) = 0$ for all $x\in [0,1]$. Thus, if any component $\textbf{i}_j \notin I_{q,m}$, i.e., if $\textbf{i}\notin \mathcal{I}_{q,m}$, then $\phi_{\textbf{i},m}^q$ vanishes on $\Omega$. 
    
    We claim that for each $q$ we have $|I_{q,m}| = (m+1)$ and thus $|\mathcal{I}_{q,m}| = (m+1)^d$. Since the length of the interval
    \begin{equation}\label{interval-equation}
        \left(-\frac{1+\delta}{2} - q\delta,~m + \frac{1+\delta}{2}-q\delta\right)
    \end{equation}
    is $m+1+\delta > m+1$ it must contain at least $m+1$ integers. To show that it always contains exactly $m+1$ integers, we note that for $q = d$ the interval starts at $0$ and thus contains exactly $m+1$ integers. Since all other intervals are shifts of this by a multiple of $\delta$ and the length of these open intervals is $m+1+\delta$, it follows that all shifts also contain exactly $m+1$ integers. 
    
    We write
    \begin{equation}\label{Phi-q,m-equation}
        \Phi_{q,m}(x) := (\phi^q_{\textbf{i},m}(x))_{\textbf{i}\in\mathcal{I}_{q,m}}\in \mathbb{R}^{(m+1)^d}.
    \end{equation}
    These functions can be expressed using ReLU networks as follows:
    
    Each of the functions $\phi_k(m(x_j - q\delta))$ for $q=0,...,2dn$, $j=1,...,d$, and $k\in I_{q,m}$ can be expressed using a ReLU network of width $W = 2$ and depth $L = 3$ (see \eqref{definition-phi-i1}). Hence a network whose output consists of all $d(m+1)$ of these functions can be represented by a ReLU network of width $W = 2d(m+1)$ and depth $L = 3$. By adding a final layer of width $(m+1)^d$ (with an application of ReLU at the end), the function $\Phi_{q,m}$ can be represented (see \eqref{definition-of-tilde-phi}).

    The next step is to reduce the embedding dimension using the digit encoding described previously in \eqref{digit-encoding-equation}. To this end, we add a linear layer to the end of each of these networks to obtain functions
    \begin{equation}
        \Phi_{q}(x) := \sum_{\textbf{i}\in\mathcal{I}_{q,m}} \phi^q_{\textbf{i},m}(x)(n+1)^{|\textbf{i}|} \in \mathbb{R},
    \end{equation}
    where in this case $|\textbf{i}|$ is simply the index of $\textbf{i}$ in some (fixed) enumeration of the set $\mathcal{I}_{q,m}$. Each of these functions can be represented by a ReLU network with three hidden layers of width $W = 2d(m+1)$ and one hidden layer with width $(m+1)^d$.
    
    We now define the function $\Phi$ in \eqref{deep-sets-equation-bound} by 
    \begin{equation}
        \Phi(x) = \begin{pmatrix}
            \Phi_{0}(x)\\
            \Phi_{1}(x)\\
            \vdots\\
            \Phi_{2dn}(x)
        \end{pmatrix}\in \mathbb{R}^{2dn+1} = \mathbb{R}^N.
    \end{equation}
    By concatenating the networks representing each $\Phi_q$, this function can be expressed using a ReLU network with three hidden layers of width $W = 2d(m+1)(2dn + 1)$ and one hidden layer of width $(2dn + 1)(m+1)^d$.

    Next, we describe how to construct the network $\rho:\mathbb{R}^N\rightarrow\mathbb{R}$. We first note that by adjusting the bias on the output, we may assume without loss of generality that the value of $f$ at the midpoint of $\Omega^n$ (namely at the point where every entry is $\frac{1}{2}$) is equal to $0$. It then follows that for any $X\in \Omega^n$ we have 
    \begin{equation}\label{a-priori-bound-on-f}
        |f(X)| \leq \omega(f,1/2)\leq m\omega\left(f,\frac{1}{2m}\right).\end{equation}
        Now consider a single partition $\Pi_q$ for $q=0,...,2dn$. For $\textbf{i}\in \mathcal{I}_{q,m}$ we consider the rectangle
        \begin{equation}\label{eqn:36}
            \mathcal{R}_{\textbf{i},q} := \left(q\delta\textbf{1} + \frac{1}{m}\prod_{j=1}^d \tilde{I}_{\textbf{i}_j}\right)\cap \Omega,
        \end{equation}
        and we denote by $z_{\textbf{i},q}$ the center point of this rectangle. Note that for each point $x\in \mathcal{R}_{\textbf{i},q}$ we have $\|x - z_{\textbf{i},q}\|_\infty \leq 1/2m$.
        
        If the partition $\Pi_q$ is good for an input $X = (x_1,...,x_n)\in \Omega^n$, then \begin{equation}\label{eqn:vector}
            \tau_{q,j}(X) := \sum_{i\in P_j} \Phi_{q}(x_i)
        \end{equation}
        is an integer of the form
        \begin{equation}
        \tau_{q,j}(X) = \sum_{\textbf{i}\in\mathcal{I}_{q,m}} \tau_{q,j}(X)_{\textbf{i}}(n+1)^{|\textbf{i}|}
        \end{equation}
        whose digits in base $(n+1)$ are given by
        \begin{equation}
            \tau_{q,j}(X)_\textbf{i} = \left|\left\{i\in P_j:~x_i\in \mathcal{R}_{\textbf{i},q}\right\}\right|.
        \end{equation} 
        This holds since if $\Pi_q$ is good for $X$, then by construction $\phi^q_{\textbf{i},m}(x)\in \{0,1\}$ for each $i=1,...,n$ and $\textbf{i}\in \mathcal{I}_{q,m}$, and $\phi^q_{\textbf{i},m}(x) = 1$ iff
        $$
            x_i - q\delta\textbf{1} \in \frac{1}{m}\prod_{j=1}^d \tilde{I}_{\textbf{i}_j}.
        $$
        Since the total number of points in each rectangle is at most $n$, the digits in base $(n+1)$ add without carrying when summing the $\Phi_q(x_i)$.

        The first step in the network $\rho$ is to extract the digits $\tau_{q,j}(X)_\textbf{i}$. This can be done using a well-known construction with a `bit-extraction' network of width $W = C(n+1)$ and depth $L = C(m+1)^d$ for an absolute constant $C$ (see for instance \cite{bartlett1998almost,yarotsky2018optimal} or \cite{siegel2023optimal}, Proposition 10).
        
        Next, given an integer vector $\tau\in \mathbb{Z}^{(m+1)^d}$ satisfying $\tau_{\textbf{i}} \geq 0$ and $\sum_{\textbf{i}} \tau_\textbf{i} = n_j$ (recall that $n_j := |P_j|$), we let $Y_q(\tau)$ be a point $Y\in \Omega^{n_j}$ such that $\tau_\textbf{i}$ components of $Y$ are equal to $z_{\textbf{i},q}$ for each $\textbf{i}$. Such a $Y$ is unique up to permutation of the indices in $P_j$ and we choose $Y_q(\tau)$ to be any such point. We set $\epsilon:=\omega(f,1/2m)$ and for $\tau = (\tau_1,...,\tau_k)$ with $\tau_j\in \mathbb{Z}^{(m+1)^d}$ we define
        \begin{equation}\label{bar-f-definition}
            \bar{f}(\tau) := \epsilon\left\lfloor \frac{f(Y_q(\tau_1),...,Y_q(\tau_k))}{\epsilon}\right\rfloor,
        \end{equation}
        where the floor is rounded towards $0$. 
        
        For $\tau = (\tau_1,...,\tau_k)$ with $\tau_j\in \mathbb{Z}^{(m+1)^d}$ we say that the set $K$ realizes $\tau$ for the partition $q$ if there exists a point in $X\in K$ such that $(\tau_j)_{\textbf{i}}$ of its coordinates in the block $B_j$ lie in the rectangle $\mathcal{R}_{\textbf{i},q}$. Since the distance between the rectangles $\mathcal{R}_{\textbf{i},q}$ and $\mathcal{R}_{\textbf{j},q}$ for $\textbf{j}\neq \textbf{i}$ is $\delta/m$, it follows that the points in $K$ realizing two different $\tau$'s must be at least distance $\delta/m$ (even modulo the action of $S_P$). Hence the number of $\tau$'s which are realized by $K$ is at most $N_r(K)$ for $r = \delta / (2m) = \frac{1}{2m(2nd+1)}$.
        
        We define a function $\tilde{\rho}_q: \mathbb{R}^{(m+1)^d\times k}\rightarrow \mathbb{R}$ which satisfies
        \begin{equation}\label{rho-q-equation}
            \tilde{\rho}_q(\tau) = \bar{f}(\tau)
        \end{equation}
        for all $\tau$ which are realized by $K$. This is an interpolation problem and we can use the results of
        \cite{vardi2022on} to construct a ReLU network satisfying \eqref{rho-q-equation}. Specifically, the number of integer vectors is bounded by $N_r(K)$,
        any two such integer vectors are at least distance one apart, and each of them has length at most $n$. In addition, using equation \eqref{a-priori-bound-on-f} we see that the number of possible values for $\bar{f}(\tau)$ is $2m+1$. Thus we can apply Theorem 3.1 from \cite{vardi2022on} to see that this interpolation problem can be solved by a ReLU network with width $W = 12$ and depth
        \begin{equation}\label{depth-bound-on-L}
            L\leq C\sqrt{N_r(K)\log(n(2m+1)N_r(K))},
        \end{equation}
        for an absolute constant $C$. We then construct $\rho_q$ by composing this network with the bit-extractor network described above. 
        
        We now observe that if the partition $\Pi_q$ is good for $X\in K$, then
        \begin{equation}\label{good-X-error-control}
            \left|f(X) - \rho_q\left(\sum_{i\in P_1} \Phi_{q}(x_i),...,\sum_{i\in P_k} \Phi_{q}(x_i)\right)\right| \leq 2\omega\left(f,\frac{1}{2m}\right).
        \end{equation}
        This holds since by construction there exists a permutation $\sigma\in S_P$ such that
        \begin{equation}
        |\sigma(X) - Y_q(\tau(X))|_\infty \leq 1/2m,
        \end{equation}
        where $\tau(X) = (\tau_1(X),...,\tau_k(X))$ is the integer vector giving the number of coordinates in the block $B_j$ that lie in the rectangle $\mathcal{R}_{\textbf{i},q}$  
        (see the remark following \eqref{reordered-sigma-equation} and piece together the permutations in each block). By \eqref{bar-f-definition} we have
        \begin{equation}
            |\bar{f}(\tau) - f(Y_q(\tau))| \leq \epsilon = \omega(f,1/2m),
        \end{equation} 
        so that by the $S_P$-invariance of $f$ we have
        \begin{equation}\label{key-inequality}
        \begin{split}
            \left|f(X) - \tilde{\rho}_q\left(\tau(X)\right)\right| &\leq |f(\sigma(X)) - f(Y_q(\tau(X))|\\
            &+ |f(Y_q((X)) - \bar{f}(\tau(X))|\\
            &\leq 2\omega\left(f,\frac{1}{2m}\right).
        \end{split}
        \end{equation}
        On the other hand, if $\Pi_q$ is bad for $X$, then we have no control on the error in \eqref{good-X-error-control}.

        We now obtain a new network by applying each $\rho_q$ to the corresponding input in parallel, i.e.,
        \begin{equation}
            \widetilde{\rho}\begin{pmatrix}
                \tau_0\\
                \tau_1\\
                \tau_2\\
                \vdots\\
                \tau_{2dn}
            \end{pmatrix} = \begin{pmatrix}
                \rho_0(\tau_0)\\
                \rho_1(\tau_1)\\
                \rho_2(\tau_2)\\
                \vdots\\
                \rho_{2dn}(\tau_{2dn})
            \end{pmatrix}.
        \end{equation}
By concatenating the constructions of each $\rho_q$, this can be done using a network with $C(m+1)^d$ layers of width $C(n+1)(2nd+1)$ followed by network with width $W = 12(2dn+1)$ and depth $L$ bounded \eqref{depth-bound-on-L}. The final network $\rho$ is obtained by taking the median of the outputs of $\widetilde{\rho}$. By Corollary 13 in \cite{siegel2023optimal} this can be done by increasing the depth of the network by $C\log(2dn+1)^2$.

Finally, we prove that \eqref{deep-sets-equation-bound} holds. Let $X = (x_1,...,x_n)\in K$. Then by construction
\begin{equation}\label{tilde-rho-equation}
    \widetilde{\rho}\left(\sum_{i\in P_1} \Phi(x_i),...,\sum_{i\in P_k} \Phi(x_i)\right) = \left[\rho_q\left(\sum_{i\in P_1} \Phi_{q,m}(x_i),...,\sum_{i\in P_k} \Phi_{q,m}(x_i)\right)\right]_{q=0}^{2dn}\in \mathbb{R}^{2dn+1}.
\end{equation}
If the partition $\Pi_q$ is good for $X$, then the bound \eqref{good-X-error-control} holds. Since the majority of the partitions $\Pi_q$ are good for any given $X\in \Omega^n$, this means that the majority of the coordinates in \eqref{tilde-rho-equation} satisfy the bound \eqref{good-X-error-control}. Consequently, the median must be sandwiched between two of these good coordinates, and thus must also satisfy the bound \eqref{good-X-error-control}. Since $\rho$ is defined to be this median, this proves that \eqref{deep-sets-equation-bound} holds.
\end{proof}

\subsection{Approximation Rates for Permutation Equivariant Functions: Sumformer}\label{subsec:equivariant}
Universal representations of group-\emph{equivariant} functions on point clouds have yielded both theoretical and practical benefits in geometric machine learning research \cite{villar2021scalars}. It allows researchers to verify that their model respects the underlying symmetries of the data. A method to extend a universal permutation \emph{invariant} deepset model to a universal permutation \emph{equivariant} model was discussed in \cite{segoluniversal} and later in the similar Sumformer architecture \cite{Sumformer}. We derive, for the first time, approximation rates for \ah continuous permutation-equivariant functions and apply these results to obtain approximation rates for the Sumformer architecture.

 We first present a Deep Sets-like representation for purely permutation \textit{equivariant} functions that accounts for the regularity of each component in the decomposition. 

\begin{cor}[Characterization of Permutation-Equivariant Functions]\label{thm:chr-equi-fun}
Let $f: \RR^{n\times d} \to \RR^{n}$ be a \ah continuous function that is equivariant to permutations. Then $f$ can be represented as
$$ f(X)_i = g(x_i,\; \{ x_j \mid j\in \{ 1,2,\ldots,n\} \setminus \{i\} \; \}) $$
where $g$ is \ah continuous. 

\end{cor}

\begin{proof}
\textit{Part I - Proof of Existence of Shared Function.} We use the equivariance property to derive symmetry conditions on the entry-wise functions of $f$. By the equivariance constraint, for any permutation $\sigma \in S_n$ it holds that
$$ f(\sigma X)=\sigma f(X) $$
$$ (f_1(\sigma X), \ldots, f_n(\sigma X)) = (f_{\sigma(1)}(X), \ldots, f_{\sigma(n)}(X)) $$

By the equality, for any $i\in [n]$,
$$ f_{i}(\sigma(X)) = f_{\sigma(i)}(X) \underset{j=\sigma(i)}{\implies} f_{\sigma^{-1}(j)}(\sigma(X)) = f_{j}(X) $$

Averaging over all permutations,
$$f_j(X) = \frac{1}{|S_n|} \sum_{\sigma \in S_n } f_{\sigma^{-1}(j)}(\sigma X) =  \frac{1}{|S_n|} \sum_{i=1}^{n} \sum_{\sigma \in Stab(j), \sigma_{0}= (ij)}  f_{i}(\sigma \sigma_{0} X)$$

$$ = \frac{1}{|S_n|} \sum_{i=1}^{n} \sum_{\sigma \in Stab(j), \sigma_{0}= (kj),\sigma_{1}=(ki)}  f_{(\sigma\sigma_0 \sigma_1)^{-1}(j)}(\sigma \sigma_{0} \sigma_1 X)$$
$$ =  \frac{1}{|S_n|} \sum_{i=1}^{n} \sum_{\sigma \in Stab(j), \sigma_{0}= (kj),\sigma_{1}=(ki)}  f_{i}(\sigma \sigma_{0}\sigma_1 X)$$
$$ = \frac{1}{|S_n|} \sum_{i=1}^{n} \sum_{\sigma \in Stab(k),\sigma_{0}=(kj)}  f_{(i)}(\sigma \sigma_0 X) = f_{k}((jk)\cdot X) $$

for any $j,k \in [n]$.

Denote $g(X):=f_{j}(X)$. As $ f_{\sigma^{-1}(j)}(\sigma(X)) = f_{j}(X)$, we have for any $\sigma \in Stab(j)$, $f_j(X)=f_j(\sigma X)$. Thus, $g(X) = g(x_j, \{ x_i \mid i=1,\ldots , n\})$.


For any $k\in [n]$, we have:
$$ f_{k}(X) = f_{j}((jk)\cdot X) = g(x_k, \{ x_i \mid i\in [n]\setminus \{j\} \})$$


\textit{Part II - Proof of Regularity.}
Let $0< \alpha\leq 1$ be the H\"older exponent and assume that $f$ is $\alpha$-H\"older continuous, that is, for every $X,Y \in \RR^{n \times d}$, $f$ satisfies $\|f(X)-f(Y)\|_2 \leq C \| X-Y \|_2^{\alpha}$, and 
denote $X^{(i)} = (x_i, x_1, x_2, \ldots, x_{i-1}, x_{i+1}, \ldots, x_n)$, which is the reordering of $X$ with the $i$-th component appearing first and then the remaining components in order, where a permutation $\sigma$ acts on $X^{(i)}$ via $\sigma \cdot X^{(i)} = (x_{\sigma(i)}, x_{\sigma(1)}, \ldots, x_{\sigma(i-1)}, x_{\sigma(i+1)}, \ldots, x_{\sigma(n)} )$.

Since $f_i(X) = g(X^{(i)})$ and $\|f_i(X)-f_i(Y)\|_2 \leq \|f(X)-f(Y)\|_2 \leq C \|X-Y \|^{\alpha}_2$, this directly gives $g$ is \ah. Therefore, for each index $i$, taking the minimum over Stab$(i)$ permutations we have

    $$\|g(X^{(i)})- g(Y^{(i)})\|_2^2 = \underset{\pi \in Stab(i)}{\min}  \|g(X^{(i)}) - g(\pi \cdot Y^{(i)}) \|_2^2   $$ $$  = \underset{\pi \in Stab(i)}{\min} \| f(X)_i- f(\pi Y)_i\|_2^2 \leq C \underset{\pi \in Stab(i)}{\min}  \|X - \pi Y \|_2^{2\alpha}$$


    Therefore if $f$ is \ah  
 with a constant of $C$ then $g$ is \ah with the same constant with respect to the $Stab(1)$-invariant metric on $\RR^{n\times d}.$

In the other direction, assume that $g$ is \ah with a constant of $C$ with respect to the Stab$(1)$-invariant norm, that is

    \begin{equation}\label{eq:stab-lipchitz}
         \|g(X) - g(Y) \|_2 \leq \frac{1}{\sqrt{n}}C \underset{\pi \in \Pi(1)}{\min} \| X - \pi Y\|_2^{\alpha}, \;   \forall X,Y
    \end{equation}
Then it holds that 

    $$ \| f(X) - f(Y) \|_2^2 = \|\left( g(X^{(1)}), g(X^{(2)}), \ldots, g(X^{(n)}) \right)  - \left( g(Y^{(1)}), g(Y^{(2)}), \ldots, g(Y^{(n)}) \right)\|_2^2  $$

    $$ =\sum_{i=1}^{n}\|g(X^{(i)}) - g(Y^{(i)}) \|_2^2 \underset{\cref{eq:stab-lipchitz}}{\leq } \frac{1}{n}C^{2}\sum_{i=1}^{n} \underset{\pi \in \Pi(1)}{\min}\| (1\; i)X - \pi(1\; i)Y \|_2^{2\alpha} \leq C^2 \|X - Y \|_2^{2\alpha}$$
    as each element in the average in smaller than the quantity on the right-hand side of the inequality. Then $f$ is \ah as well.
\end{proof}

As a consequence of Corollary \ref{thm:chr-equi-fun} and Theorem \ref{main-deep-sets-approximation-theorem}, we obtain approximation rates for \ah continuous permutation-equivariant functions.
\begin{cor}[Universal Approximation of Permutation Equivariant Functions]\label{thm:equi-fun}
    Let $\Omega = [0,1]^d$, and suppose that $f:\Omega^n\rightarrow \mathbb{R}^n$ is an $S_n$-equivariant function. Let $m \geq 1$ be an integer and set $N = 2dn+1$. Then there exist ReLU neural networks $\Phi:\mathbb{R}^d\rightarrow \mathbb{R}^N$ and $\rho:\mathbb{R}^N\rightarrow \mathbb{R}$ with the following structure:
    \begin{itemize}
        \item The network $\Phi$ has three hidden layers of width $W = 2d(2dn+1)(m+1)$, followed by one hidden layer of width $(m+1)^d$, and an output layer of width $2nd+1$.
		\item The network $\rho$ consists of $C(m+1)^d$ layers of width $C(n+1)$ followed by $L$ layers of width $W = 12(2dn+1)$, where
		$$r = \frac{1}{2m(2nd+1)}, \quad L \leq C\max\{\sqrt{N_r(K)\log[n(2m+1)N_r(K)]},\log(2nd+1)^2\},$$
	and $C$ is an absolute constant.
    \end{itemize}
    such that
    \begin{equation}\label{deep-sets-equation-bound-snir}
        \left|f(X)_i - \rho\left(\Phi(x_i),\sum_{j\in [n]\setminus \{i\}}\Phi(x_j)\right)\right| \leq 2\omega\left(f,\frac{1}{2m}\right)
    \end{equation}
    for every $X = (x_1,...,x_n)\in \Omega^n$. 
\end{cor}

\begin{proof}[Proof of Corollary \ref{thm:equi-fun}]
By Corollary \ref{thm:chr-equi-fun}, each entry $i$ of an \ah continuous, permutation equivariant function is defined by a shared \ah invariant function corresponding to a partition $P_1=\{i\}$ and  $P_2=[n]\setminus\{i\}$. We immediately obtain the desired result by Theorem \ref{main-deep-sets-approximation-theorem}.   
\end{proof}

In the next subsection, we explore whether the recently popular Transformer \cite{vaswani2017attention} architecture can approximate continuous equivariant functions with the same approximation rates as Corollary \ref{thm:equi-fun}. We first detail the mechanics of the Transformer architecture and then proceed to derive the desired approximation rates.

\subsection{Definition of the Transformer Architecture}
We first define the components of the Transformer architecture, as referenced in \cite{Sumformer}.

\begin{definition}[Attention Head (Def. 2.1 in \cite{Sumformer})]

Let $W_{Q}, W_{K}, W_{V} \in \mathbb{R}^{d\times d}$ be weight matrices. A \textdef{self-attention head} is a function $\textnormal{AttHead}: \RR^{n\times d} \rightarrow \RR^{n\times d}$ with
\begin{equation}
    \textnormal{AttHead}(X) := \mathrm{softmax} \left( \frac{(XW_{Q})(XW_{K})^{\top}}{\sqrt{d}} \right) XW_{V}
\end{equation}
where $\mathrm{softmax}$ is the softmax function applied row-wise.
\end{definition}

\begin{definition}[Attention Layer (Def. 2.2 in \cite{Sumformer})]

Let $h\in\NN$, let 
$\textnormal{AttHead}_1,$  $ \textnormal{AttHead}_2,..., \textnormal{AttHead}_h$ be attention heads, and 
let $W_O \in \RR^{hd \times d}$. A \textdef{(multi-head) self-attention layer} $\textnormal{Att}: \RR^{n \times d} \rightarrow \RR^{n \times d}$ is defined as
\begin{equation}
    \textnormal{Att}(X) := [\textnormal{AttHead}_1(X), ..., \textnormal{AttHead}_h(X)] W_O
\end{equation}
\end{definition}

\begin{definition}[Transformer Block (Def. 2.3 in \cite{Sumformer})]\label{def:trans-block}
A \textdef{Transformer}    \textdef{block} $\textnormal{Block}: \RR^{n \times d} \rightarrow \RR^{n \times d}$ is an attention layer $\textnormal{Att}$ followed by a fully-connected feed forward layer $FC: \RR^d \rightarrow \RR^d$ with \textdef{residual connections}:
\begin{equation}
    \textnormal{Block}(X) := X + FC(X + \textnormal{Att}(X))
\end{equation}
where $FC$ is applied row-wise.
\end{definition}

\begin{definition}[Transformer Network (Def. 2.4 in \cite{Sumformer})]\label{def:transfor}
Let $l \in \NN$ and $\textnormal{Block}_1,$  $ \textnormal{Block}_2,..., \textnormal{Block}_l$ be Transformer blocks. A \textdef{Transformer network} $\cT: \RR^{n \times d} \rightarrow \RR^{n \times d}$ is a composition of Transformer blocks:
\begin{equation}
    \cT(X) := (\textnormal{Block}_l \circ \dots \circ \textnormal{Block}_1)(X)
\end{equation}
\end{definition}

Using the definitions above, we prove the Transformer architecture can approximate permutation equivariant functions with identical approximation rates as in Corollary \ref{thm:equi-fun}.

\begin{cor}[Transformer \cite{vaswani2017attention} Approximation Rates]\label{cor:tr}
    Let $\Omega = [0,1]^d$, and suppose that $f:\Omega^n\rightarrow \mathbb{R}^n$ is an $S_n$ equivariant function. Let $m \geq 1$ be an integer and set $N = (2dn+1)(m+1)^d$. Then there exists a 2-Block Transformer network  $\cT$, with the fully-connected layers (see Definition \ref{def:trans-block}) satisfying $FC_1 = \rho$ and $FC_2 = \phi$, where $\rho$ and $\phi$ are ReLU Neural Networks which satisfy the conditions of Corollary \ref{thm:equi-fun}, up to $\phi$ having an additional linear layer of width $2N + d + 1$ and up to $\rho$ being preceded by a linear projection from $\mathbb{R}^{2N+d+1}$ to $\mathbb{R}^{2N}$ that extracts the last $2N$ coordinates, such that

    \begin{equation}\label{deep-sets-equation-bound-snir-trans}
        \left|f(X)_i - \cT(X)_i \right| \leq 2\omega\left(f,\frac{1}{2m}\right)
    \end{equation}
    for every $X = (x_1,...,x_n)\in \Omega^n$ and $i\in [n]$.

\end{cor}
\begin{proof}[Proof of Corollary \ref{cor:tr}]
We now show that a Transformer Network (Definition 4) can be constructed to compute the expression $\rho\left(\Phi(x_i),\sum_{j\in [n]\setminus \{i\}}\Phi(x_j)\right)$ from Corollary \ref{thm:equi-fun}. We follow the constructive proof from \cite{Sumformer} (Appendix B), modifying it to compute the `omit-one' sum directly. The target expression is $ \rho(\Phi(x_i), \Sigma_{\Phi}^{\setminus i})$, where $\Sigma_{\Phi}^{\setminus i} = \sum_{j \neq i} \Phi(x_j)$.

We begin by leveraging the feed-forward network $\Phi: \RR^d \rightarrow \RR^N$, guaranteed to exist by Corollary \ref{thm:equi-fun}. We modify $\Phi$ by enlarging its width by $d+1$ so that the identity $x_i$ and a constant $1$ are passed through alongside the original output. We then append a final linear layer that produces both $\Phi(x_i)$ and $-\Phi(x_i)$, yielding the augmented map $\Hat{\Phi}(x) = [1, x_i, \Phi(x_i), - \Phi(x_i)] \in \RR^{1+d + 2N}$.

We then form an augmented matrix $X_1 \in \RR^{n \times (1+d+2N)}$ where the $i$-th row is defined as:
$$ (X_1)_i = \Hat{\Phi}(x_i) $$
This $X_1$ serves as the input to the Transformer Block (Definition 3). The block's first operation is $X_1 + \text{Att}(X_1)$. We configure one Attention Head (Definition 1) to compute the global sum $\Sigma_{\Phi} = \sum_{k=1}^n \Phi(x_k)$.

Let $D = 1+d+2N$ be the dimension of the augmented space. We set the attention head's key/query dimension $d_k = 1$. Let $e_1 \in \RR^D$ be the first standard basis vector. We define the query and key matrices $W_Q, W_K \in \RR^{D \times 1}$ as:
$$ W_Q = W_K = e_1 = [1, 0, \dots, 0]^T $$
The resulting query and key matrices are $Q = X_1 W_Q = \bm{1}_n \in \RR^{n \times 1}$ and $K = X_1 W_K = \bm{1}_n \in \RR^{n \times 1}$. The attention matrix $A \in \RR^{n \times n}$ is then:
$$ A = \mathrm{softmax}( Q K^T / \sqrt{d} ) = \mathrm{softmax}( \bm{1}_n \bm{1}_n^T / \sqrt{d} ) = \mathrm{softmax}( \bm{1}_{n \times n} / \sqrt{d} ) = \frac{1}{n} \bm{1}_{n \times n} $$
where $\mathrm{softmax}$ is the row-wise softmax function.
We define the value matrix $W_V \in \RR^{D \times D}$ using a block structure based on the partition $D = (1+d) + N + N$:
$$ W_V = \begin{bmatrix}
\mathbf{0}_{(1+d) \times (1+d+N)} & \mathbf{0}_{(1+d) \times N} \\
\mathbf{0}_{N \times (1+d+N)} & n \cdot I_N \\
\mathbf{0}_{N \times (1+d+N)} & \mathbf{0}_{N \times N}
\end{bmatrix} $$
The row-wise application $(X_1 W_V)_i$ selects the $\Phi(x_i)$ components (from the third block of $(X_1)_i$), scales them by $n$, and places them in the final $N$ columns.
The output of the attention head $\text{AttHead}(X_1) = A (X_1 W_V)$ has an $i$-th row given by:
$$ (\text{AttHead}(X_1))_i = \frac{1}{n} \sum_{j=1}^n (X_1 W_V)_j = \frac{1}{n} \sum_{j=1}^n [\mathbf{0}_{1+d+N}, n \cdot \Phi(x_j)] = [\mathbf{0}_{1+d+N}, \Sigma_{\Phi}] $$
where $\Sigma_{\Phi} = \sum_{j=1}^n \Phi(x_j)$.
The Transformer Block (Definition 3) then adds the residual $X_1$ to this attention output. For the $i$-th row, this operation yields:
\begin{align*}
    (X_1 + \text{AttHead}(X_1))_i &= (X_1)_i + (\text{AttHead}(X_1))_i \\
    &= [1, x_i, \Phi(x_i), -\Phi(x_i)] + [\mathbf{0}_{1+d+N}, \Sigma_{\Phi}] \\
    &= [1, x_i, \Phi(x_i), \Sigma_{\Phi}- \Phi(x_i)]
\end{align*}
Let $Z_i = [1, x_i, \Phi(x_i), \Sigma_{\Phi} - \Phi(x_i)]$ be this resulting vector. This vector $Z_i$ contains both the `self' embedding $\Phi(x_i)$ and the `omit-one' sum $\Sigma_{\Phi}^{\setminus i} = \Sigma_{\Phi} - \Phi(x_i)$.

Finally, we apply a feed-forward network $\rho$ (with the specified width and depth) that takes $Z_i$ as input. After projecting $Z_i$ naturally to the last $2N$ coordinates, $\pi_{2N}(Z_i)$, which  now contains both $\Phi(x_i)$ and $\Sigma_{\Phi}^{\setminus i}$, we plug into $\rho$. Thus, the transformer can express the target function $\rho(\Phi(x_i), \Sigma_{\Phi} - \Phi(x_i))$. This is precisely the expression from Corollary \ref{thm:equi-fun}, Equation \eqref{deep-sets-equation-bound-snir}.

\end{proof}

\subsection{Approximation Rates for Non-Invariant Functions}
 Theorem \ref{main-deep-sets-approximation-theorem} also gives us a new approximation result for approximation of non-invariant functions. A non-invariant function $f$ can be seen as an $S_P$ invariant function in the special case where 
 $$P=\left\{ \{1\}, \{2\},\ldots,\{n\}              \right\} .$$
 Setting $d=1$ in Theorem \ref{main-deep-sets-approximation-theorem} we obtain the following corollary
 \begin{cor}\label{cor:intrinsic}
Let $f:[0,1]^n\rightarrow \mathbb{R}$ and let $K\subset \Omega^n$. Denote by $N_r(K)$ the smallest number of balls (in the $\ell_\infty$-norm) of radius $r$ required to cover $K$.
	
	Let $m \geq 1$ be an integer and set $N = 2n+1$. Then there exist ReLU neural networks $\Phi:\mathbb{R}\rightarrow \mathbb{R}^N$ and $\rho:\mathbb{R}^{N\times k}\rightarrow \mathbb{R}$ satisfyfing that 	for every $X = (x_1,...,x_n)\in K$,
		\begin{equation}\label{deep-sets-equation-bound-non-inv}
		\left|f(X) - \rho\left(\Phi(x_1),...,\Phi(x_n)\right)\right| \leq 2\omega\left(f,\frac{1}{2m}\right),
	\end{equation}
	where
	\begin{itemize}
		\item The network $\Phi$ has three hidden layers of width $W = 2(2n+1)(m+1)$, followed by one hidden layer of width $m+1$, and an output layer of width $2n+1$.
		\item The network $\rho$ consists of $C(m+1)$ layers of width $C(n+1)$ followed by $L$ layers of width $W = 12(2n+1)$, where
		$$r = \frac{1}{2m(2n+1)}, \quad L \leq C\max\{\sqrt{N_r(K)\log[n(2m+1)N_r(K)]},\log(2nd+1)^2\},$$
	and $C$ is an absolute constant.
	\end{itemize}

 In particular, if there exists some $C_1,d_{\text{intrinsic}} $ such that  $N_\delta(K)\leq C_1\delta^{-d_{\text{intrinsic}}}$ for all positive $\delta$, then there exists some $C_2=C_2(C_1,n,d_{\text{intrinsic}}) $ such that any function $f$ which is $\alpha$-H\"older with $|f|_\alpha=1 $ can be approximated to $\epsilon$ accuracy by a ReLU network whose total number of parameters is bounded by $C_2(1/\epsilon)^{\frac{d_{\text{intrinsic}}}{2\alpha}}$.
 \end{cor}

 The rate of $(1/\epsilon)^{\frac{d_{\text{intrinsic}}}{2\alpha}} $ obtained in the corollary  is an improvement by a factor of two (inside the exponent) over the $\sim \epsilon^{-\frac{d_{\text{intrinsic}}}{\alpha}}$ factor obtained in previous work \cite{nakada2020adaptive}. 

\section{Approximation Rates for Functions Invariant to Permutations and Rigid Motions}\label{sec:rotation} 
The purpose of this section is to discuss approximation rates for functions which are invariant to both permutations and rigid motions. We will divide our discussion into two subsections. The first subsection will discuss the approximation rates of (non-invariant) ReLU neural networks. The second subsection will utilize these results, together with our approximation rates for Deep Sets,  to prove approximation rates for invariant models which are constructed by applying a rotation frame to a Deep Sets permutation invariant  backbone. 
\subsection{ReLU approximation rates}
Our goal in this subsection is to determine approximation rates for (non-invariant) ReLU networks. Our first and most technically involved result will pertain to functions which are invariant only to rotations: 
\begin{theorem}[ReLU approximation rates for $O(d)$ invariant functions]\label{thm:reluO}
Let $d$ and $n\geq d$ be natural numbers, and let $\alpha \in (0,1] $. Let $K\subseteq \RR^{d\times n}$ be the compact set 
$$K=\{(x_1,\ldots,x_n)\in \RR^{d\times n}| \quad \|x_j\| \leq 1, \forall j=1,\ldots,n\} .$$
Then there exist $W=W(d,\alpha),L=L(d,\alpha),q=q(d,\alpha)$, such that, for all $\epsilon>0$, and any  $O(d)$ invariant function $f:\RR^{d\times n}\to \RR$ which is $\alpha$-H\"older with $|f|_\alpha=1$, can be approximated to $\epsilon$ accuracy on the set $K$ by a neural network with width $\leq W$ and depth $\leq L(1/\epsilon)^{\frac{nd-\dim(O(d))}{2\alpha}}\left(\log(1/\epsilon)\right)^q$. 
\end{theorem}
As a corollary for this theorem, we can easily get approximation rates for $E(d)$  invariant functions:

\begin{cor}[ReLU approximation rates for $E(d)$ invariant functions]\label{cor:reluE}
Let $d$ and $n\geq (d+1)$ be natural numbers and let $\alpha\in (0,1] $. Let $K_{1/2}\subseteq \RR^{d\times n}$ be the compact set 
$$K_{1/2}=\{(x_1,\ldots,x_n)\in \RR^{d\times n}| \quad \|x_j\| \leq \frac{1}{2}, \forall j=1,\ldots,n\} .$$
Then there exist $W=W(d,\alpha),L=L(d,\alpha),q=q(d,\alpha)$, such that, for all $\epsilon>0$, and any $E(d)$ invariant function $f:\RR^{d\times n} \to \RR $ which is $\alpha$-H\"older with $|f|_\alpha=1 $,  can be approximated to $\epsilon$ accuracy on the set $K_{1/2}$ by a neural network with width $\leq W$ and depth $\leq L(1/\epsilon)^{\frac{nd-\dim(E(d))}{2\alpha}}\left(\log(1/\epsilon)\right)^q$.
\end{cor}

\begin{proof}[Proof of Corollary]
For any point cloud $X$ we define $T$ to be the mapping which translates all coordinates of $X$ by the last coordinate, namely 
$$T(X)=(x_1-x_n,\ldots,x_{n-1}-x_n,0). $$
The image of $K_{1/2}$ under $T$ is contained in the set $$\hat K=\{(x_1,\ldots,x_n)| \|x_j\| \leq 1, x_n=0 \} , $$
which can be identified with the $d\times (n-1)$ dimensional set $K$ from the theorem, by ignoring the last coordinate which is always zero. By the theorem, on $\hat K $ we can approximate $f$ to $\epsilon$ accuracy using a ReLU network $\psi$ with width $\leq W$ and depth $\leq L(1/\epsilon)^{\frac{(n-1)d-\dim(O(d))}{2\alpha}}\left(\log(1/\epsilon)\right)^p$, which is the approximation rate advertised in the corollary, since $(n-1)d-\dim(O(d))=nd-\dim(E(d)) $. Thus, for any $X\in K_{1/2}$, we have 
$$|f(X)-\psi(TX)\|=|f(TX)-\psi(TX)\|\leq \epsilon ,$$
and since $\psi(TX)$ is a composition of the neural network $\psi$ with the linear function $T$, it too is a ReLU network of the dimensions advertised in the corollary.
\end{proof}

Finally, we can achieve a similar result for functions invariant not only  to rigid motions, but also  to permutations. Since the permutation group is finite, the approximation rates we would expect in this case are identical, and since in this setting the functions we are interested in are invariant to rigid motions, the approximation rates we have already obtained still apply. In other words we obtain another trivial corollary: 

\begin{cor}\label{cor:perm_also}
The approximation rates for approximating $O(d)$ invariant functions by ReLU networks, stated in  Theorem \ref{thm:reluO}, are also applicable for $O(d)\times S_n $ invariant functions. The rates for approximating $E(d)$ invariant functions by ReLU networks, stated in  Corollary \ref{cor:reluE}, are also applicable for $E(d)\times S_n $ invariant functions.
\end{cor}
The rest of this section is devoted to proving Theorem \ref{thm:reluO}. At very high level, the idea of the proof is that, one can choose $d$ points from the point cloud, e.g., $x_1,\ldots,x_d$, and, assuming they are linearly independent, one can use the Gram Schmidt procedure to turn them into an orthonormal basis $u_1,\ldots,u_d$ of $\RR^d$, with the property that $x_1,\ldots,x_k$ and $u_1,\ldots,u_k$ have the same span, for all $1\leq k \leq d$. Taking $U$ to be the orthonormal matrix whose rows are $u_1,\ldots,u_d$, we have that $f(X)=f(UX)$, and that since $\langle x_i,u_j \rangle=0 $ for all $i<j$, the point cloud $UX$ will reside in a linear subspace 
$$V=\{X\in \RR^{d\times n}| \quad x_{i}(j)=0, \forall 1\leq i<j \leq d \} $$
whose dimension is exactly the dimension of the quotient space, namely  $n\cdot d -\frac{d^2-d}{2} $. At a high level, this reduces the problem to the problem of approximation on this subspace $V$, and on this subspace we can get the approximation rates we want since its dimension is the dimension of the quotient space. 

The challenges in implementing this strategy are that we do not always have that $x_1,\ldots,x_d$ are linearly independent. In some cases, there could be a different $d$-tuple which is linearly independent. In other, it may be that $X$ is not of rank $d$. We will handle these challenges in a way similar to the frame constructions in \cite{dym2024equivariant,pozdnyakov2023smooth}. 

\begin{remark}
A unique orthogonal basis for $\RR^d$ can be determined in an $SO(d)$ invariant way, can actually be obtained using only $d-1$ vectors and their (generalized) vector product. We choose to use $d$ vectors to define the basis since it simplifies the analysis. 
\end{remark}

\subsubsection{Preliminaries I: Gram Schmidt Stability}
As preliminaries for the proof, we will introduce two lemmas regarding the stability of the Gram-Schmidt procedure. 

The Gram-Schmidt procedure is based on repeated application of the two functions $p$ and $N$ defined by 
\begin{equation}\label{eq:Nandp}
p(x)=x, \quad p(x;u_1,\ldots,u_k)=x-\sum_{i=1}^k \langle x,u_i \rangle u_i  ,\quad
N(y)=
\frac{y}{\|y\|}.
 \end{equation}
In terms of these function, the Gram Schmidt procedure for  given $k\leq d$ linearly independent vectors $x_1,\ldots,x_k$ is defined recursively via 
\begin{align*}
v_j&=p(x_j;u_1,\ldots,u_{j-1})\\
u_j&=N(v_j)
\end{align*}

We begin with two lemmas on the stability of the functions $p$ and $N$ to perturbations of their input:
\begin{lem}\label{lem:p}
For $0< \eta \leq 1$, let $u_1,\ldots,u_k \in \RR^d$ be an orthonormal family, and let $\tilde u_1,\ldots,\tilde u_k\in \RR^d$ such that 
$$\|u_j-\tilde u_j\|\le \eta, \forall j=1,\ldots,k $$
Denote 
\begin{align*}
v&=p(x|u_1,\ldots,u_k)\\
\tilde v&=p(x|\tilde u_1,\ldots,\tilde u_k).
\end{align*}
Then for all $x\in \RR^d$ of norm $\|x\|\leq 1$, we have 
$$\|v-\tilde v\|\leq 3k\eta .$$
\end{lem}
\begin{proof}
Using the triangle and Cauchy-Schwarz inequalities, we obtain
\begin{align*}
\|v-\tilde v\|&\leq \sum_{i=1}^k \| \langle x,u_i\rangle u_i-\langle x,\tilde{u}_i\rangle \tilde{u}_i\| \\
&\leq \sum_{i=1}^k \left(  \|\langle x,u_i \rangle u_i-\langle x,\tilde{u}_i \rangle u_i\|+ \|\langle x,\tilde u_i \rangle u_i-\langle x,\tilde{u}_i \rangle \tilde u_i\| \right) \\
&\leq \sum_{i=1}^k \left( |\langle x,u_i-\tilde u_i \rangle| \|u_i\|+|\langle x,\tilde u_i 
\rangle|\|u_i -\tilde u_i\|\right)\\
&\leq \sum_{i=1}^k \left(\|x\| \|u_i\| \|u_i-\tilde u_i\|+\|x\| \|\tilde{u}_i\| \|u_i-\tilde u_i\| \right)\\
&\leq 3k\eta,
\end{align*}
where for the last inequality we used the fact that $\|\tilde u_i\|\leq 2 $ since 
$$\|\tilde u_i\|\leq \|u_i\|+\|\tilde u_i-u_i\|\leq 1+\eta\leq 2$$
\end{proof}
\begin{lem}\label{lem:N}
If $z,\tilde{z}\in \RR^d$ and $\|z\|\geq \alpha > 0$ and $\|z-\tilde z\|\leq \beta<\alpha/2$, then 
$$\|N(z)-N(\tilde{z}) \|\leq \frac{4\beta}{\alpha} $$
\end{lem}
\begin{proof}
We have 
\begin{align*}
\|N(z)-N(\tilde z)\|&=\frac{1}{\|z\|\cdot \|\tilde z\|}\cdot \| \|\tilde z\|\cdot z-\|z\| \cdot \tilde z \|\\
&\leq \frac{1}{\|z\|\cdot \|\tilde z\|}  \cdot \left(\| \|\tilde z\| \cdot z-\|z\|\cdot z\|+\|\|z\| \cdot  z-\|z\| \cdot \tilde z \| \right)\\
&\leq \frac{1}{\|z\|\cdot \|\tilde z\|} \cdot 2\|z\| \cdot \|z-\tilde z\|\\
&\leq \frac{4 \beta}{\alpha}
\end{align*}
where we use the fact that 
$$\|\tilde z\| \geq \|z\|-\|\tilde z -z\| \geq \alpha -\beta \geq \frac{\alpha}{2} $$
\end{proof}

\subsubsection{Preliminaries II: Super approximation and the root function}
Important for our proof is the fact that the Gram-Schmidt procedure is composed of arithmetic operations which can be `super-approximated' by ReLU networks. The function $p$ is a polynomial, composed of a finite number of multiplication and addition operations. Addition can be implemented exactly by neural networks. For multiplication, Proposition 3 in \cite{yarotsky2017error} states that the function $(x,y)\mapsto x\cdot y $ can be approximated to $\epsilon$ accuracy on $[-M,M]$ by a ReLU network with fixed width and depth proportional to $\log(1/\epsilon)+\log(M)$. Using this, it can be shown that $p$ can be approximated to $\epsilon$ accuracy with only $\sim \log(1/\epsilon)$ parameters.

The function $N(z)=\frac{z}{\|z\|}$ is the composition of the function $q(z,s)=z/s$ with the function $\|z\|=\sqrt{\|z\|^2} $. The function $\|z\|^2$ is a polynomial and hence we have super approximation for this function. Similarly, as shown in \cite{telgarsky_rational}, Theorem 1.1, 
the rational function $q$ can be approximated to $\epsilon$ accuracy on $[-1,1]^d\times [\eta,1] $, by a neural network with fixed width and depth which is bounded by $\ln^7(1/\eta) \ln^3(1/\epsilon) $. The only remaining gap is similar super approximation results for the root function, which to the best of our knowledge have not been proven so far. We address this gap with the following lemma, which follows using similar techniques to the argument in \cite{telgarsky_rational}.
\begin{lem}\label{power-lemma}
	Let $\alpha > 0$. For every $0 < \epsilon \leq 1/4$, there exists a deep ReLU neural network $\mathcal{N}\in \Upsilon^{21,L}(\mathbb{R})$ with depth $L \leq C|\log(\epsilon)|^3\log(|\log(\epsilon)|)$, such that
	\begin{equation}
		\sup_{x\in [0,1]}|\mathcal{N}(x) - x^\alpha| < C\epsilon.
	\end{equation}
    Note that the logarithms here are taken base $2$, and the constants depend only upon $\alpha$.
\end{lem}
\begin{proof}
	We begin by showing that there is a network $\mathcal{N}_0\in \Upsilon^{17,L}$ with $L \leq C|\log(\epsilon)^2\log( |\log(\epsilon)|)|$ such that (here the constant $C$ depends upon $\alpha$)
	\begin{equation}
		\sup_{x\in [1/2,1]}|\mathcal{N}_0(x) - x^\alpha| < \epsilon.
	\end{equation}
	To do this, we utilize the Taylor expansion of $x^{\alpha}$ about $1$,
	\begin{equation}
		x^\alpha = \sum_{n=0}^\infty \frac{\alpha(\alpha - 1)\cdots (\alpha - n+1)}{n!}(x-1)^n.
	\end{equation}
	Using Taylor's estimate for the remainder, and the fact that
    \begin{equation}
		\left|\frac{\alpha(\alpha - 1)\cdots (\alpha - n + 1)}{n!}\right| \leq C_\alpha,
	\end{equation}
	uniformly in $n$, we obtain the following bound for the error incurred by truncating this series after $k$ terms:
	\begin{equation}
		\sup_{x\in [1/2,1]}\left|x^\alpha - \sum_{n=0}^k \frac{\alpha(\alpha - 1)\cdots (\alpha - n + 1)}{n!}(x-1)^n\right| \leq C_\alpha2^{-k}.
	\end{equation}
	Moreover, the coefficients in this expansion are also bounded by $C_\alpha$.
	 Since $|x - 1| \leq 1/2$, we can use the construction of \cite{yarotsky2017error} (see for instance also Proposition 3 in \cite{siegel2023optimal}) and repeated squaring, to approximate the term
	\begin{equation}
		\frac{\alpha(\alpha - 1)\cdots (\alpha - n + 1)}{n!}(x-1)^n
	\end{equation}
	to error $\delta$ using a neural network in $\Upsilon^{15,L}(\mathbb{R})$ with $L \leq C\log(n)|\log(\delta)|$.
    
    We now choose $k$ large enough so that $C_\alpha 2^{-k} < \epsilon/2$. This requires $k = O(|\log(\epsilon)|)$. Then we apply this result for each term $n=0,...,k$ with $\delta_n = \epsilon 2^{-{n+1}}$. Summing up these networks, we obtain the desired network $\mathcal{N}_0$ which satisfies
    \begin{equation}
		|\mathcal{N}_0(x) - x^\alpha| < \epsilon/2 + \epsilon\sum_{n=0}^k 2^{-n+1} \leq \epsilon.
	\end{equation}
    Counting the size of the network, we see that $\mathcal{N}_0 \in \Upsilon^{17,L}$ (the summing operation requires the width to be increased by $2$) with
    \begin{equation}
    \begin{split}
        L &\leq C\sum_{n=0}^k \log(n)|\log(\epsilon 2^{-{n+1}})|\\ &\leq 
        C\sum_{n=0}^k \log(n)(|\log(\epsilon)| + (n+1))\\
        &\leq C(k^2\log(k) + k\log(k)|\log(\epsilon)|) \leq C|\log(\epsilon)^2\log(|\log(\epsilon)|)|.
    \end{split}
    \end{equation}
	Next, for each integer $j \geq 0$, we construct a piecewise linear function
	\begin{equation}
		\phi_j(x) = \begin{cases}
			1 & 2^{-j/3-1} \leq x \leq 2^{-j/3}\\
			0 & x < 2^{-j/3 - 1.1}\\
			0 & x > 2^{-j/3 + 0.1}.
		\end{cases}
	\end{equation}
	We clearly have $\phi_j\in \Upsilon^{4,1}$. We now construct the functions
	\begin{equation}\label{Nt-equation}
		N_t(x) = \sum_{n=0}^{k}2^{-\alpha n}\phi_{3n+t}(x)\mathcal{N}_0(2^n x)
	\end{equation}
	for $t=0,1,2$. Notice that the $\phi_j$ have the following property:
    \begin{equation}\label{phis-overlap-property}
        \text{for all $x\in [0,1]$ there are at least two indices $t$ such that $\phi_{3n+t}(x)\in \{0,1\}$ for all $n$.}
    \end{equation}

    Utilizing the product  and sum network constructions again we see that each $N_t$ can be approximated to accuracy $\epsilon$ with a network $\mathcal{N}_t\in\Upsilon^{19,L}$ with $L\leq Ck|\log(\epsilon)\log(|\log(\epsilon)|)|$. Moreover, we remark that the product network construction is exact if one term is $0$ (see \cite{yarotsky2017error} or \cite{siegel2023optimal} for example). Hence the network $\mathcal{N}_t$ equals $0$ on the interval $[0,2^{k/3-1.1}]$. 
    
    Choosing $k$ so that  $2^{k\alpha/3} \leq \epsilon$, from which it follows that $k = O(|\log(\epsilon)|)$ (for a constant depending upon $\alpha$), the networks $\mathcal{N}_t$ will have depth $L\leq C|\log(\epsilon)^3\log(|\log(\epsilon)|)|$, and moreover will satisfy 
    \begin{equation}
        \text{for all $x\in [0,1]$},~~|\mathcal{N}_t(x) - x| \leq C\epsilon,~\text{for at least two indices $t$.}
    \end{equation}
    When $x < 2^{k/3-1.1}$ this follows since $|x^\alpha| = |x|^\alpha \leq 2^{k\alpha/3} \leq \epsilon$ (by our choice of $k$) and $\mathcal{N}_t(x) = 0$, while for the other values of $x$ it follows from the approximation result for $\mathcal{N}_0$, the homogeneity of the function $x^\alpha$, the property \eqref{phis-overlap-property}, and the fact that the product network construction is exact if one term is $0$ (and hence only a fixed number of terms in \eqref{Nt-equation} incur a non-zero error). 
    
    Finally, taking the median of the three networks $\mathcal{N}_0,\mathcal{N}_1,\mathcal{N}_2$ gives the result (see \cite{siegel2023optimal}, Corollary 13).
	
\end{proof}
	
Now that we have efficient approximation of root functions, we can prove that the functions $p$ and $N$ used to construct the Gram-Schmidt procedure can be approximated efficiently, as long as the input to $N$ is bounded away from zero:
\begin{lem}\label{lem:Np}
Let $d,k$ be natural numbers and $M,\delta,\epsilon$ be positive numbers. Then there exist constants $C_1=C_1(d,k,M)$ and $C_2=C_2(d,M) $, such that:

The function  $p$, defined in \eqref{eq:Nandp}, can be approximated to $\epsilon$ accuracy on $[-M,M]^{d\times (k+1)} $ by a ReLU network with $\leq  C_1\left(\log(|\epsilon|\right)$ parameters. 

The function $N$, also defined in \eqref{eq:Nandp}, can be approximated to $\epsilon$ accuracy on $[\delta,M] $ with $\leq C\left(\log^3(|\epsilon|\cdot \log^7(|\delta|)\right)$ parameters.
\end{lem}
\begin{proof}
As mentioned above, the approximation result for $p$, which is a polynomial, appears in \cite{yarotsky2017error}.

The function $N(x)$ can be written as a composition 
$$N(x)=q(x,r\circ s(x)) $$
where 
$$q(x,t)=x/t, \quad r(t)=\sqrt t \text{ and } s(x)=\|x\|^2. $$
Combining the approximation results from \cite{yarotsky2017error} for the polynomial $s$ with our approximation results for the square root and the approximation results from \cite{telgarsky_rational}, Theorem 1.1, 
for the rational function $q$ gives us the required result. 
\end{proof}

\subsubsection{Proof of Theorem \ref{thm:reluO}}
For given $\delta>0$  we define the set 
$$\mathcal{D}_1(\delta)=\{x\in \RR^d| \quad \delta\leq \|x\|\leq 1  \} $$
For any $k\leq d$ we define recursively
$$\mathcal{D}_{k}(\delta)=\{(x_1,\ldots,x_k)\in \RR^d| (x_1,\ldots,x_{k-1})\in \mathcal{D}_{k-1}(\delta) \quad \text{ and } \|x_k\| \leq 1   \text{ and } \|p(x_k;u_1,\ldots,u_{k-1})\|\geq \delta  \}  $$
where $u_j=u_j(X)$ is the $j$-th vector obtained from $X$ during the Gram-Schmidt procedure. We prove
\begin{lem}\label{lem:u}
Fix some natural $d$.
There exist $W=W(d),L=L(d)$, such that, 
for all $k\leq d$ and  $\delta\in [0,1]$, there exist relu neural network $\tilde{u}_i:\RR^{d\times k}\to \RR^d, i=1,\ldots,k $,  with width $\leq W$ and depth $\leq L\log(1/\delta)$, such that  
$$\|\tilde{u}_i(X)-u_i(X)\|\leq \left(\frac{\delta}{20\cdot d} \right)^{d+1-k} $$
\end{lem}
\begin{proof}
Denote 
$$\eta_k=\left(\frac{\delta}{20\cdot d} \right)^{d+1-k} $$.

By induction on $k$, the cardinality of the $k$-tuple. 

For $k=1$, since $\|x_i\|\geq \delta$, we can build a neural network $\tilde N$ approximating $N(x_i)$ to accuracy of $\left(\frac{\delta}{20\cdot d} \right)^{d} $, with a number of parameters which depends poly-logarithmically on $\delta$ (see Lemma \ref{lem:Np}).  

Now assume correctness for $k$, and let us prove the claim for $k+1$. 

By Lemma \ref{lem:p}, we know that 
$$\|p(x_{i_{k+1}};u_{1},\ldots,u_{k})-p(x_{i_{k+1}};\tilde{u}_{1},\ldots,\tilde{u}_{k})\| \leq 3k \cdot \eta_k $$
While $p$ is not a neural network, it is a polynomial, and so we can obtain an approximation $\tilde p$ of it of order $k \eta_k$, so that 
$$\tilde z_{k+1}= \tilde p(x_{i_{k+1}};\tilde{u}_{1},\ldots,\tilde{u}_{k})$$
is a $4k\eta_k$ approximation of 
$$z_{k+1}= p(x_{i_{k+1}};u_{1},\ldots,u_{k})$$
and the number of parameters in the neural network $\tilde p$ grows like  $\log |\eta_k |\sim \log |\delta| $ 
By assumption, $\|z_{k+1}\| \geq \delta$, and so applying Lemma \ref{lem:N} with $\alpha=\delta$ and 
$$\beta=4k\eta_k\leq 4k \frac{\delta}{20d} \leq  \frac{4}{20}\cdot \delta \leq \frac{\alpha}{2}, $$
we obtain that 
$$\|N(z_{k+1})-N(\tilde{z}_{k+1})\| \leq 16k \frac{\eta_k}{\delta} $$
While $N$ is not a neural network, we can replace it with a neural network $\tilde N$ with fixed width, and depth depending logarithmically on $\delta$, such that $\tilde N(\tilde z_{k+1})$ is a $(k \eta_k)/\delta$ approximation of $N(\tilde z_{k+1}) $. All in all, we obtain 
$$\tilde u_{k+1}=\tilde N(\tilde z_{k+1}) $$
as an output of a neural network whose size depends logarithmically on $\delta$, and 
$$\|u_{k+1}-\tilde{u}_{k+1}\|\leq \frac{17k}{\delta} \eta_k \leq \frac{20 d}{\delta}\eta_k=\eta_{k+1}$$
\end{proof}

Next, for a $k$-tuple $\mathbf{i}=(i_1,\ldots,i_k)$, we define 
$$\mathcal{D}(\mathbf{i},\delta)=\{X\in K| \quad  (x_{i_1},\ldots,x_{i_k}) \in \mathcal{D}_k(\delta)\} $$

We will also allow a definition of a $0$-tuple $\mathbf{i}=\emptyset $, and we will define 
$$\mathcal{D}(\emptyset,\delta)=K. $$
We will now conclude by proving the following lemma:
\begin{lem}
Let $d$ and $n\geq d$ be natural numbers. Let $K\subseteq \RR^{d\times n}$ be the compact set 
$$K=\{(x_1,\ldots,x_n)| \quad \|x_j\| \leq 1, \forall j=1,\ldots,n\} .$$
Then there exist $W=W(d),L=L(d)$, such that, for all $k$-tuples $\mathbf{i} $ with $0\leq k \leq d $, all $\epsilon>0$, and any $O(d)$ invariant function $f:\RR^{d\times n}\to \RR$ which is $\alpha$-H\"older with $|f|_\alpha=1$,  can be approximated to $\epsilon$ accuracy on the set $\mathcal{D}(\mathbf{i},\delta)$, where
\begin{equation}\label{eq:delta} \delta \leq \min\left \{1, \frac{1}{2}\left(\frac{\epsilon}{4(\sqrt{nd}+1)^\alpha} \right)^{1/\alpha}  \right \}
\end{equation}
by a neural network with width $\leq W$ and depth $\leq L(1/\epsilon)^{\frac{nd-\dim(O(d))}{2\alpha}}\log(1/\epsilon)$
\end{lem}
We note that Theorem \ref{thm:reluO} follows from the lemma in the special case where $\mathbf{i}$ is the $0$ tuple $\mathbf{i}=\emptyset$.
\begin{proof}
We prove by induction on $k$, the cardinality of the $k$ tuple $i_1,\ldots,i_k $. 

\textbf{Base case: $k=d$} We denote by $u_1,\ldots,u_d$  the vectors obtained from $x_{i_1},\ldots,x_{i_d}$ via Gram-Schmidt. Throughout we supress the dependence of these vectors on $X$ for readability.  By Lemma \ref{lem:u}, we know that, for all $i=1,\ldots,d$ there is a neural network $\tilde u_i=\tilde u_i(X) $ such that 
\begin{equation}\label{eq:utilde}\|u_i-\tilde u_i\| \leq \frac{\delta}{20\cdot d}. \end{equation}
Let $U$ be the orthonormal matrix whose rows are the vectors $u_1,\ldots,u_d$, and $\tilde{U}$ be the matrix whose rows are $\tilde u_i$. Then $UX$ maps $X$ to the vector space 
$$V(\mathbf{i})=\{X\in \RR^{d\times n}| \quad x_{i_j}(s)=0, \forall 1\leq j<s \leq d \} $$
whose dimension is exactly the dimension of the quotient space, namely  $n\cdot d -\frac{d^2-d}{2} $. According to Corollary \ref{cor:intrinsic}, we can approximate $f$ on the intersection of $V_{\mathbf{i}}$ with the compact set 
$$\tilde K=\{X| \|X\|_F\leq 2\sqrt{nd}+1 \}$$
to $\epsilon/4$ accuracy, with a ReLU network $\tilde f$ with $\sim (1/\epsilon)^{\frac{n\cdot d -\frac{d^2-d}{2} }{2\alpha}}$
parameters. Note that $K$ is strictly contained in $\tilde K$ as the norm of each column of $X\in K$ is at most $1$. 

The matrix $\tilde U$ is ``almost a rotation'', in the sense that 
\begin{equation}\label{eq:Utilde}
\|U-\tilde U\|_F^2=\sum_{i=1}^d \|u_i-\tilde u_i\|^2\leq d\left(\frac{\delta}{20\cdot d} \right)^2 \leq \delta^2\end{equation}
 We can approximate the product $\tilde{U}X $ to $\delta$ accuracy, by a neural network $M(\tilde U,X)$ whose size depends logaritmically on $\delta$. We then have 
\begin{align*}\|UX-M(\tilde U,X)\|_F
&\leq \|(U-\tilde U)X\|_F+\|\tilde UX-M(\tilde U,X)\|_F\\
&\leq \|U-\tilde U\|_F\|X\|_F+\delta\\
&\leq \delta(\sqrt{nd}+1)
\end{align*}

Our approximation for $f$ on the domain will be $\tilde f\circ \pi \circ M\left(\tilde{U}(X), X \right) $, where $\pi$ is the projection onto the space $V_{\mathbf{i}}$. We have

\begin{align*}
|f(X)-\tilde f\circ \pi \circ M\left(\tilde{U}(X), X \right) |&=|f\circ \pi(UX))-\tilde f\circ \pi \circ M\left(\tilde{U}(X), X \right) |\\
&\leq |f\circ \pi(UX))- f\circ \pi  \circ M\left(\tilde{U}(X), X \right)|\\
&\quad +| f\circ \pi  \circ M\left(\tilde{U}(X), X \right)-\tilde f\circ \pi  \circ M\left(\tilde{U}(X), X \right)|\\
&\stackrel{(*)}{\leq} \|UX-M(\tilde{U}(X),X)\|_F^\alpha+\frac{\epsilon}{4}\\
&\stackrel{(**)}{\leq} \epsilon/2
\end{align*}
where for (*) we used the fact that  $f$ is $\alpha$-H\"older and $\pi$ is projection, as well as the fact that $\tilde f $ is an $\epsilon/4$ approximation of $f$ on $V_{\mathbf{i}}\cap \tilde{K}$, and  $M(\tilde U(X),X)$ is in $\tilde K $, since
$$\|M(\tilde U(X),X)\|_F\leq \|UX\|_F+\|UX-M(\tilde U(X),X)\|_F\leq \|X\|_F+\delta(\sqrt{nd}+1)\leq 2\sqrt {nd}+1.$$
For (**) we used \eqref{eq:delta}.

\textbf{Induction step} Now assume that the claim is true for all $k+1$ tuples, we want to prove correctness for $k$ tuples. Let $\mathbf{i}$ be a $k$-tuple, and for an index $j\in [n]$ denote the $k+1$ tuple obtained by concatenating $j$ to $\mathbf{i}$ by $\mathbf{i}\circ j $. To construct an approximation for $f$ on $V_{\mathbf{i}}\cap K $, we note that for every index $j$, there is a neural network $\tilde f_j$ of the desired size, giving an $\epsilon$ approximation of $f$ on the intersection of $K$ with
$$\mathcal{D}(\mathbf{i}\circ j,\delta)=\{X\in V_{\mathbf{i}}, \|p(x_j;u_{i_1},\ldots,u_{i_{k}})\|\geq \delta  \} $$
The set $\mathcal{D}(\mathbf{i},\delta)$ is covered by the union of all these sets,  with the set 
$$\mathcal{E}(\mathbf{i},\delta)=\{X\in V_{\mathbf{i}}| \quad \|p(x_j;u_{i_1},\ldots,u_{i_{k}})\|\leq 2\delta, \quad \forall j=1,\ldots,n  \} .$$
On the set $\mathcal{E}(\mathbf{i},\delta)$, we can approximate $f$ with a neural network to $3\epsilon/4$ accuracy, with the desired approximation rate, as follows:

First, we use Lemma \ref{lem:u} to construct neural networks $\tilde u_{i_1}(X),\ldots,\tilde u_{i_k}(X) $ with a number of parameters logarithmic in $\delta$, such that \eqref{eq:utilde} for all $X\in \mathcal{E}(\mathbf{i},\delta)$. We then denote by $U,\tilde U$ the $d\times d$ matrix
$$U=\begin{pmatrix}
- u_{i_1} -\\
\vdots\\
-u_{i_k}-\\
-0_d-\\
\vdots\\
-0_d-
\end{pmatrix}, \quad  
\tilde{U}=\begin{pmatrix}
-\tilde u_{i_1}-\\
\vdots\\
-\tilde u_{i_k}-\\
-0_d-\\
\vdots\\
-0_d-
\end{pmatrix}$$
and we have that \eqref{eq:Utilde} holds. 
and we denote by $\hat U$ the orthonormal matrix obtained by replacing the zero rows of $U$ by orthonormal rows. We then have that $f(X)=f(\hat U X)$ for all $X\in \RR^{d\times n}$. Moreover, for all $X\in \mathcal{E}(\mathbf{i},\delta)$ we have 
\begin{align*}
|f(X)-f(UX)|&=|f(\hat{U}X)-f(UX)|\leq \|\hat{U}X-UX\|_F^{\alpha} \\
&= \left(\sum_{j=1}^n \|(\hat{U}-U)x_j\|_2^2\right)^{\alpha/2} \leq  \left(\sum_{j=1}^n \|p(x_j;u_{i_1},\ldots,u_{i_k})\|^{2}\right)^{\alpha/2}\\
&\leq (n(2\delta)^2)^{\alpha/2} = (2\delta)^{\alpha}n^{\alpha/2}
\end{align*}
By the definition of $\delta$ in \eqref{eq:delta}, since $\sqrt{n}\leq \sqrt{nd}+1$, we have $(2\delta)^{\alpha}n^{\alpha/2}\leq \epsilon/4$, so $f(UX)$ is an $\epsilon/4$ approximation of $f(X)$. Now, we note that $UX$ is in the vector space 
$$\mathcal{V}(\mathbf{i})=\{X\in \RR^{d\times n}| x_{i_j}(s)=0, \forall 1\leq j<s \leq k \text{ and } x_t(s)=0, \forall t=1,\ldots,n, \quad s=k+1,\ldots,d \} $$
We note that for any completion of $\mathbf{i} $ to a $k$-index with distinct elements $\mathbf{j} $, we have $\mathcal{V}(\mathbf{i})\subseteq \mathcal{V}(\mathbf{j}) $, and therefore $f$ can be approximated on $\mathcal{V}(\mathbf{i})\cap \tilde K $ by a neural network $\tilde f_0$ to $\epsilon/4$ accuracy with the approximation rate we desire, using Corollary \ref{cor:intrinsic}. We denote the projection onto $\mathcal{V}(\mathbf{i}) $ by $\pi$, and we then obtain for all $X\in \mathcal{E}(\mathbf{i},\delta)$ that 
\begin{align*}
|f(X)-\tilde f_0 \circ \pi \circ M(\tilde U(X), X)|&\leq |f(X)-f(UX)|+|f\circ \pi (UX)-\tilde f_0 \circ \pi \circ M(\tilde U(X), X)|\\
&\leq \epsilon/4+|f\circ \pi (UX)-f \circ \pi \circ M(\tilde U(X), X)|\\
& \quad +|f\circ \pi \circ M(\tilde U(X), X)-\tilde f_0\circ \pi \circ M(\tilde U(X), X)|\\
&\leq \epsilon/4+\|UX-M(\tilde U(X),X) \|_F^\alpha+\epsilon/4\\
&\leq \epsilon/4+\delta^\alpha (\sqrt{nd}+1)^\alpha+\epsilon/4\leq 3\epsilon/4
\end{align*}
where the last inequality uses \eqref{eq:delta}, which gives $\delta^\alpha(\sqrt{nd}+1)^\alpha\leq \epsilon/4$. All in all we obtain a $3\epsilon/4$ approximation for $f$ on $\mathcal{E}(\mathbf{i},\delta) $. 

Now that we have an approximation of $f$ by ReLU networks on a collection of sets which cover the domain $\mathcal{D}(\mathbf{i},\delta)$, we will combine these ReLU networks to a global approximation using a partition of unity argument. For $0<a<b$ we define the piecewise linear function 
\begin{equation}
		\phi_{a,b}(x) = \begin{cases}
			0 &  x \leq a\\
			\frac{x-a}{b-a} & x\in [a,b]\\
			1 & x \geq  b.
		\end{cases}
	\end{equation}
This can be realized as a ReLU network of depth 1 and width 2.  Another ingredient we will need for the partition of unity is a ReLU network approximation of the function $p(x;u_{i_1},\ldots,u_{i_k}) $. Using Lemma \ref{lem:u} we approximate the $u_i$ by a neural network $\tilde u_i$ with an error of at most $\frac{\delta}{20\cdot d}$. By \ref{lem:p} we have 
$$| \|p(x;u_{i_1},\ldots,u_{i_k})\|-\|p(x;\tilde{u}_{i_1},\ldots,\tilde{u}_{i_k})\| | \leq \|p(x;u_{i_1},\ldots,u_{i_k})-p(x;\tilde{u}_{i_1},\ldots,\tilde{u}_{i_k})\|\leq 3d \frac{\delta}{20\cdot d}=\frac{3\delta}{20}.$$ 

Since $\|p\|$ is a composition of a polynomial with the root function, we can approximate it  to accuracy of  $\delta/20$, with a neural network $\nu$ with $\sim \log(1/\delta)$ parameters, and we will have 
\begin{align*}| \|p(x_j;u_{i_1},\ldots,u_{i_k})\|-\nu(x_j;\tilde u_{i_1},\ldots,\tilde u_{i_k})|&\leq | \|p(x;u_{i_1},\ldots,u_{i_k})\|-\|p(x;\tilde{u}_{i_1},\ldots,\tilde{u}_{i_k})\| |\\
& \quad \quad +| \|p(x;\tilde{u}_{i_1},\ldots,\tilde{u}_{i_k})\|-\nu(x;\tilde{u}_{i_1},\ldots,\tilde{u}_{i_k}) |  \\
&\leq \frac{\delta}{20}+\frac{3\delta}{20}=\frac{\delta}{5} . 
\end{align*}

Let us now relabel $\nu(x_j;\tilde u_{i_1},\ldots,\tilde u_{i_k}) $ as 
$$\nu_j(X):=\nu(x_j;\tilde u_{i_1},\ldots,\tilde u_{i_k}) $$

We will approximate $f$ by the function $\hat f$ defined by 
\begin{align*}
\psi(X)&=\min_{1\leq j \leq n} \phi_{a,b}\left(\nu_j(X)\right) , \text{ where } a=\frac{8\delta}{5}, b=\frac{9\delta}{5} \\
\phi_j(X)&=\phi_{a,b}\left( \nu_j(X)\right) , \text{ where } a=\frac{6\delta}{5}, b=\frac{7\delta}{5} \\
N(X)&=\sum_{j=1}^n  \phi_j(X)\\
\hat f(X)&=(1-\psi(X))\tilde{f}_0(X)+\psi(X)\cdot \frac{1}{N(X)} \sum_{j=1}^n  \phi_j(X)\tilde{f}_j(X)
\end{align*}
We note that the coefficient $1-\psi(X) $ multiplying $\tilde f_0(X)$ is not zero, only when $\nu_j(X)<\frac{9\delta}{5} $ for all $j$.  This implies that all projections of the $j$-th coordinate have norm of at most $2\delta$, so $X\in \mathcal{E}(\mathbf{i},\delta)$, and so for all such $X$, the value $f(X)$ is approximated by $\tilde f_0(X)$ to $3\epsilon/4$ accuracy. 

When $\psi(X)$ is not zero, we can deduce that there exists some $j$ for which $\nu_j(X)>\frac{8\delta}{5}$ and thus the norm of  $p(x_j;x_{i_1},\ldots,x_{i_k})$ is larger than $\frac{7\delta}{5} $. Thus, for this $j$ we will have that $\phi_j(X)=1 $ and that $X\in \mathcal{D}(\mathbf{i},\delta) $. In particular, this means that $N(X)\geq 1 $ whenever $\psi(X)>0$ and so the division by $N(X)$ is well defined. In fact, for any $j$ for which $\phi_j(X)>0$ we have that $\nu_j(X)>\frac{6\delta}{5}$ and thus the norm of  $p(x_j;x_{i_1},\ldots,x_{i_k})$ is larger than $\delta $, and so $X\in \mathcal{D}(\mathbf{i},\delta) $. It follows that when the coefficient $\psi(X)\cdot \phi_j(X) $ is not zero, we know that  $\tilde f_j(X)$ is an $\epsilon$ approximation of $f(X)$ (by the inductive hypothesis). It follows that
\begin{align*}
|f(X)-\hat f(X)|&=|(1-\psi(X))(f(X)-\tilde f_0(X))+\psi(X) \frac{1}{N(X)}\sum_{j=1}^n \phi_j(X)(\tilde f_j(X)-f(X))| \\
&\leq (1-\psi(X))\cdot 3\epsilon/4+\psi(X)\cdot\epsilon\leq \epsilon.
\end{align*}
Finally, we note that $\hat f$ is defined by product and summations or neural networks, and this can be $\epsilon/2$ approximated using neural networks which approximate multiplications, with an additional cost of $\sim \log(1/\epsilon)$. This concludes the proof of the lemma, and the proof of the theorem.
\end{proof}

\subsection{Approximation rates for rotation  invariant function spaces using frames}\label{subsec:frames}
Frames are a popular recent notion in the design of invariant models for point sets and other datatypes. For a group $G$  acting on a collection of point sets, we define a frame to be any mapping $X\mapsto \mu_X$ which assigns to each point set $X$ a measure $\mu_X$ over $G$. This is used to define a generalized averaging operation on functions $f$ defined on point sets is, via 
\begin{equation}\label{eq:frame_avg}
\mathcal{I}[f](X)=\int_G f(g^{-1}X)d\mu_X(g).
\end{equation}
To ensure that the function $\mathcal{I}[f]$ is invariant, one can require that the mapping $X\mapsto \mu_X$ is $G$-equivariant, or weakly $G$ equivariant \cite{dym2024equivariant} (though our approximation results will not require these equivariance assumptions). One simple example of a frame is the mapping which takes $X$ to the Haar measure on $G$, in which case the frames performs a standard group averaging procedure. Another family of examples are canonicalization algorithms, where each $\mu_X$ is supported on a single element, see e.g., \cite{minimal_frame}.

From a practical perspective, an advantage of frame based models is they can apply a successful non-invariant model $f_\theta$ to the data, and enforce invariance via the generalized averaging operation to obtain an invariant model $\mathcal{I}[f]  $. From the approximation perspective, we have a similar advantage here. Namely, approximation rate we have for a ReLU neural network $f_\theta$, will be inherited by frame-averaged $\mathcal{I}[f]$. In particular we can adapt Corollary \ref{cor:reluE} and Corollary \ref{cor:perm_also} to obtain

\begin{prop}[Frame approximation rates for $E(d)\times S_n$ invariant functions]\label{prop:reluE}
Let $d$ and $n\geq d+1$ be natural numbers and let $\alpha\in (0,1] $. Let  $G=E(d)\times S_n$, and let $\mathcal I$ be a frame averaging operator as in \eqref{eq:frame_avg}. Let $K_{1/4}\subseteq \RR^{d\times n}$ be the compact set 
$$K_{1/4}=\{(x_1,\ldots,x_n)\in \RR^{d\times n}| \quad \|x_j\| \leq \frac{1}{4}, \forall j=1,\ldots,n\} .$$
Then there exist $W=W(d,\alpha),L=L(d,\alpha),q=q(d,\alpha)$, such that, for all $\epsilon>0$, and any $E(d)$ invariant function $f:\RR^{d\times n} \to \RR $ which is $\alpha$-H\"older with $|f|_\alpha=1 $,  can be approximated to $\epsilon$ accuracy on the set $K_{1/4}$ by $\mathcal{I}[\psi] $, where $\psi$ is a neural network with width $\leq W$ and depth $\leq L(1/\epsilon)^{\frac{nd-\dim(E(d))}{2\alpha}}\left(\log(1/\epsilon)\right)^q$.
\end{prop}

\begin{proof}
The claim follows almost immediately from Corollary \ref{cor:reluE} and Corollary \ref{cor:perm_also}. The only technicality that needs to be addressed is, that for all $X\in K_{1/4}$, we need to ensure that all possible $g^{-1}(X)$ encountered in the averaging operation in \eqref{eq:frame_avg} will reside in some known compact set. 

To this effect, we define the centralization operator
\begin{equation}\label{eq:cent}
\cent(X)=X-\frac{1}{n}X1_n1_n^T \end{equation}
which translates $X$ so that the average of its columns is zero. One can easily verify that
$$\cent(gX)\in K_{1/2}, \quad \forall X \in K_{1/4}, g\in G .$$
Using Corollary \ref{cor:reluE} we can approximate $f$ on $K_{1/2} $ to $\epsilon$ accuracy by a neural network $\phi$ with the advertised number of parameters. We then have that the ReLU network $\psi:=\phi \circ \cent $ satisfies our claim because 
 for all $X$ in $K_{1/4} $,
\begin{align*}
|f(X)-\mathcal{I}[\psi](X)|&=|\mathcal{I}[f](X)-\mathcal{I}[\psi](X)|\\
&=\left| \int \left[f(g^{-1}X)-\psi(g^{-1}X) \right]d\mu_X(g)  \right| \\
&\leq \int \left|f(g^{-1}X)-\psi(g^{-1}X) \right|d\mu_X(g)\\
&=\int \left|f\circ \cent(g^{-1}X)-\phi\circ \cent(g^{-1}X) \right|d\mu_X(g)<\epsilon 
\end{align*}
\end{proof}

\paragraph{Frame averaging with a Deep Sets backbone}
While frame averaging and canonicalization has been applied directly to MLPs (see for example \cite{friedmann2025canonnetcanonicalorderingcurvature}, it seems that a more popular approach is to define an averaging approach over rotations and applying it to a permutation invariant ``backbone'' like Deep Sets  to achieve permutation invariance. This type of strategy is discussed in \cite{pozdnyakov2023smooth,punyframe,dym2024equivariant}. In this setting we cannot currently show that any frame averaging operator will have the desired approximation rates, but we can show this is the case for some examples. 

For simplicity, let us consider point clouds with $d=2$, with a frame similar to the one  defined in \cite{pozdnyakov2023smooth,dym2024equivariant}. The frame averaging operator here is of the form 
$$\mathcal{I}[f](X)=\sum_{i=1}^n w_i(X)f(R_i^{-1}\cent(X)), $$
where $R_i^{-1}=R_i^{-1}(X) $ is defined to be the rotation in $SO(2)$ which rotates $X_i$ to the positive $x$ axis. The weights $w_i$ are defined so that they vanish when $x_i=0$ (as in this case $R_i^{-1}(X)$ is not well defined. In any case, we see that 
$R_i^{-1}\cent(X)$ resides in the subspace 
$$V_i=\{X\in \RR^{2\times n}| \quad X1_n=0_d, \quad X_{i,2}=0 \} $$
which is of dimension 
$$\dim(V_i)=2n-3=\dim(V_i)-\dim(E(2)).  $$
Accordingly, we can obtain $\epsilon$ approximation forH\"older functions $f$, using a model $\mathcal{I}[\psi]$, where $\psi$ is a Deep Sets model, using our Corollary \ref{deep-sets-corollary}. 

For point clouds with dimension $d\geq 3$ a similar frame can be constructed, and similar approximation rates can be obtained, we refer to \cite{pozdnyakov2023smooth,dym2024equivariant} for more details. 

A more popular $E(d)$ averaging procedure uses the singular value decomposition of the point cloud. This method was first introduced in \cite{punyframe} and was also used for FAENet \cite{faenet}. A given point cloud $X$ is first centralized to attain translation invariance, and then the rotations are fixed by considering the SVD decomposition of $\cent(X)$, namely 
\begin{equation*}
	\cent(X)=UDV^T 
	\end{equation*}
where $U\in O(d) $, the matrix $D\in \RR^{d\times d}$ is a diagonal matrix with non-negative entries (the singular values of $\cent(X)$), and $V\in \RR^{n\times d} $ is a matrix whose $d\leq n$ columns are orthonormal. In the case where the singular values  are distinct, this decomposition is unique up to multiplication of $U$ and $V$ by a diagonal matrix $S $ whose diagonal entries are in $\{-1,1\}$. Thus there are in total $2^d$ (in applications typically $d=3$ so $2^d=8$) different possible decompositions 
\begin{equation}\label{eq:svd}
\cent(X)=U_jDV_j^T, \quad j=1,\ldots,2^d .
\end{equation}
The frame suggested by \cite{punyframe} is then defined by 
\begin{equation}\label{eq:puny}
\mathcal{I}[f](X)=2^{-d}\sum_{j=1}^{2^d} f(U_j^{-1}\cent(X)). \end{equation}
It is also shown in \cite{punyframe} that the model obtained by applying this frame to Deep Sets if invariant to both rigid motions and permutations, and moreover can approximate all such functions uniformly on some compact set $K$. This claim  requires that $K$ is contained in 
$$\Omega_{d,n}=\{X\in \RR^{d\times n}| \text{ the singular values of } \cent(X) \text{ are pairwise distinct. }\} $$
so that the frame in \eqref{eq:puny} is well defined. We now prove approximation rates for this method:
\begin{prop}
Let $d,n$ be natural numbers, such that $n\geq d+1 $. Let $\alpha \in (0,1])$ and let $\Omega_{d,n}' $ be the set 
$$\Omega_{d,n}'=\{ X\in \RR^{d\times n}| \text{ the singular values of } \cent(X) \text{ are pairwise distinct, and } \|\cent(X)\|_F\leq 1  \} .$$
 Then there exists a constant $C=C(d,n,\alpha) $, such that, for every $\epsilon>0$ and every $\alpha$-H\"older function $f$ with $|f|_\alpha=1 $, there exists a Deep Sets model $\psi_{\mathrm{Deep Sets}} $ with 
 $$P\leq C\epsilon^{-\frac{nd-\dim(E(d))}{2\alpha}}(1+\log(|\epsilon|) $$
parameters, such that 
$$|f(X)-\mathcal{I}[\psi_{\mathrm{Deep Sets}} ](X)|\leq \epsilon, \quad \forall X\in \Omega_{d,n}'.$$ 
\end{prop}
\begin{proof}
Using \eqref{eq:svd} we see that the expressions $U_j^{-1}\cent(X) $  are equal to 
$$ U_j^{-1}\cent(X) =DV_j^T$$
where $D$ is a diagonal matrix whose diagonal entries are the singular values of $\cent(X) $, and $V_j^T$ is a $d\times n $ matrix whose $d$ rows are orthonormal. Also note that $V_j^T1_n=0 $ since $\cent(X)1_n=0 $, so that the rows of $V_j^T$ are orthogonal to $1_n$. The space of all such matrices $V_j^T$ can be identified with the Stiefel manifold 
$$\mathcal{V}_{d,n-1}=\{V\in \RR^{d\times (n-1)}| \quad \text{ the rows of } V \text{ are orthonormal }\}. $$
Explicitly,  $V_j^T$ can be written as $V_j^T= V\circ T_0$, where $V\in \mathcal{V}_{d,n-1} $ and $T_0:\RR^n\to \RR^{n-1} $ is a fixed linear mapping defined by setting $T_01_n=0_{n-1}$, and by fixing some isometric mapping of the orthogonal complement of $1_n$  to $\RR^{n-1} $. 

Now, for all $X\in \Omega_{d,n}'$, using the argument we have seen above, we see that $U_j^{-1}\cent(X) $ is in the image of the smooth function 
$$h:\RR^d \oplus \mathcal{V}_{d,n-1} \to \RR^{d\times n}, \quad h(\lambda,V)=\mathrm{diag}(\lambda)\circ V \circ T_0  $$
Since the frobenius norm, and thus the maximal singular value, of $\cent(X) $, is bounded by $1$, we know in fact that $U_j^{-1}\cent(X) $ is in the image of $h$ when restricted to the compact set $K=[0,1]^d \oplus \mathcal{V}_{d,n-1}  $. This is a compact subset of the manifold  $\RR^d \oplus \mathcal{V}_{d,n-1} $, whose dimension is given by 
$$D:=\dim\left(\RR^d \oplus \mathcal{V}_{d,n-1} \right)=d+(n-1)d-\frac{d(d+1)}{2}=nd-\frac{d(d+1)}{2}=nd-\mathrm{dim}(E(d)). $$
It follows that for an appropriate $C_1>0 $,
$$N_r(K)\leq C_1r^{-D}, \quad \forall r>0. $$
and since $h$ is smooth, it is $L$-Lipschitz on $K$ for an appropriate positive $L$. We can then produce a cover of $h(K)$ by balls of radius $r$, by applying $h$ to a cover of $K$ by balls of radius $r/L $. Thus 
$$N_r(h(K))\leq C_1L^Dr^{-D}, \quad \forall r>0. $$
We can then apply Corollary \ref{deep-sets-corollary} to obtain an $\epsilon$ approximation for $f$ on $h(K) $ by a Deep Sets network $\psi_{\mathrm{Deep Sets}} $ with $\sim \epsilon^{-\frac{D}{2\alpha}}(1+\log|\epsilon|)$ parameters. It follows that for all $X\in \Omega_{d,n}'$. The notation $1_n$ (respectively $0_n$) stands for the vector in $\RR^n$ whose entries are all $1$ (respectively $0$). 
\begin{align*}
|f(X)-\mathcal{I}[\psi_{\mathrm{Deep Sets}} ](X)|&=|\mathcal{I}[f](X)-\mathcal{I}[\psi_{\mathrm{Deep Sets}} ](X)|\\
&=2^{-d}|\sum_{j=1}^{2^d} f(U_j^{-1}\cent(X))-\psi_{\mathrm{Deep Sets}}(U_j^{-1}\cent(X))|\\
&\leq 2^{-d}\sum_{j=1}^{2^d} |f(U_j^{-1}\cent(X))-\psi_{\mathrm{Deep Sets}}(U_j^{-1}\cent(X))|\leq \epsilon
\end{align*}
\end{proof}

\section{Bi-Lipschitz invariant models}\label{sec:bilip}
In recent years there has been growing interest in designing bi-Lipschitz invariant models. Such models provide an invariant embedding of the quotient space with a bounded metric distortion, which can be useful for metric based tasks like nearest neighbor search \cite{cahill2024towards} or for avoiding oversquashing in graph neural networks \cite{fswgnn}. Examples of recent papers discussing bi-Lipschitzness include \cite{balan2022permutation,amir2024fsw} for bi-Lipschitz permutation invariant learning on point clouds, \cite{Mizrachi} for the alternating group,  extensions to rotations in \cite{sverdlov2025toward}, and \cite{davidson2024,fswgnn} in the context of graph neural networks, as well as the more general max filtering approach \cite{cahill2024group}.  We will next formally define bi-Lipschitz invariant models, and then obtain approximation rates for these models.  

As discussed in Section \ref{sec:prelim}, we will consider the action on $V=\RR^N$ of a group $G$ which is a closed subgroup of the group of isometries of $\RR^N$. We can then define the metric $d_G$ on the quotient space $V/G $ simply by quotienting by the group, as in \eqref{eq:dG}. 


An invariant function $E:V\to \RR^m$ can be uniquely identified with a function $\tilde E:V \to \RR^m$ satisfying the equation 
$$\tilde E([x])=E(x), \quad \forall x\in V .$$

The following simple but useful lemma shows that for invariant functions $f$, H\"older stability with respect to $\|\cdot \|_V $ is equivalent to H\"older stability with respect to the metric on the quotient space. 
\begin{lem}\label{lem:holder}
Let $V=\RR^N$ and let $G$ be a closed subgroup of the isometry group of $V$. Let $(Y,d_Y)$ be some metric space,  let $f:V\to Y $ be an invariant function and let $\tilde f:V/G\to Y$ be the induced function on the quotient space. Then $f$ is $\alpha$-H\"older if, and only if, $\tilde f$ is $\alpha$-H\"older. Moreover, $|f|_\alpha=|\tilde f|_\alpha $.
\end{lem}
\begin{proof}
Assume $\tilde f$ is $\alpha$-H\"older. Then for all $x,x'\in V$, we have 
$$d_Y(f(x),f(x'))=d_Y(\tilde f([x]),\tilde f([x']))\leq |\tilde f|_\alpha d^{\alpha}(x,x')\leq |\tilde f|_\alpha \|x-x'\|_2^\alpha $$
and thus $f$ is $\alpha$-H\"older and $|f|_\alpha \leq |\tilde f|_\alpha  $.

In the other direction, assume $f$ is $\alpha$-H\"older. Choose some $x,x'\in V$. there exists some $g\in G$ for which 
$$d_G([x],[x'])=\|x-gx'\|_2 .$$
For this $g$, we have 
$$d_Y(\tilde f([x]),\tilde f([x']))=d_Y(f(x),f(gx'))\leq |f|_\alpha \|x-gx'\|_V^\alpha=|f|_\alpha d(x,x')^\alpha. $$
This shows that $\tilde f $ is $\alpha$-H\"older and $|f|_\alpha \geq |\tilde f|_\alpha  $, which  concludes the proof.
\end{proof}
We say that $E$ 
is an invariant bi-Lipschitz mapping, with constants $0<L_1,L_2$, if  $\tilde E$ is invertible, and the mapping $\tilde E$ and $\tilde E^{-1}$ are $L_1$ and $L_2$ Lipschitz, respectively. Equivalently, this means that
$$L_2^{-1} d_G([x],[y])\leq \|E(x)-E(y)\|_\infty\leq L_1 d([x],[y]), \quad \forall x,y\in V. $$

A trivial approximation result for Lipschitz invariant functions $f$ can be obtained by observing that $f=f\circ E^{-1}\circ E$, and $f\circ E^{-1} $ is Lipschitz. However, this will lead to an approximation rate whose exponent depends on the embedding dimension $m$, which is generally larger than the true ``intrinsic dimension'' $\dim(K/G)$.

Using Corollary \ref{cor:intrinsic}, we can prove the following:

\begin{theorem}\label{thm:bi_lip}
Let $V=\RR^N$ and let $G$ be a closed subgroup of the isometry group of $V$. Let $E:V\to \RR^m$ be an invariant  bi-Lipschitz mapping with constants $L_1,L_2$. Let $C_1>0$ and $\dint\geq 2 $, and let $K\subseteq V$ be a $G$ invariant subset of $V$, such that $K/G$ is a  compact subset of  $V/G$ of diameter $1$, and such that
$$N_r(K/G)\leq C_1r^{-\dint}, \quad \forall r> 0 .$$
Then there exists a constant $P=P(N,G,L_1,L_2,C_1,\dint)$, such that, for all $\epsilon>0$, and every
$f:V \to \RR$  which is invariant and $\alpha$-H\"older with $|f|_\alpha=1 $,  there exists a ReLU networks $\psi$ expressible with  $\leq  P(1/\epsilon)^{\frac{\dint}{2\alpha}}\log(1/\epsilon) $ parameters, such that 
$$|f(x)-\psi\circ E(x)|<\epsilon, \quad  \forall x \in K $$
\end{theorem}
\begin{proof}
Choose $\epsilon>0$. 

Let $\tilde E$ and $\tilde f$ denote the maps on the quotient space $V/G$, induced by the invariant maps $E,f$.

Since $K/G$ has diameter $1$, we have that $\tilde E(K/G)=E(K) $ is of diameter $\leq L_1 $. We also have that $\tilde f \circ \tilde E^{-1} $ is $\alpha$-H\"older because 
 $$|\tilde f \circ \tilde E^{-1}(x)- \tilde f \circ \tilde E^{-1}(y)|\leq  d_G^\alpha( \tilde E^{-1}(x), \tilde E^{-1}(y))\leq L_2^\alpha \|x-y\|^\alpha, \forall x,y\in \RR^m.$$
We note that this argument implicitly uses the fact that the function $\tilde E^{-1}$, originally a $L_2$ Lipschitz function defined on $E(V)\subseteq \RR^m $, can be extended to a $L_2$ Lipschitz function on all of $\RR^m$, by Kirszbraun theorem.
 
 We can now apply Corollary \ref{cor:intrinsic} to the function $\tilde f \circ \tilde E^{-1} $ on the compact set $E(K) $. Technically, to do this we need to scale $E$ by an affine transformation to ensure that $E(K)\subseteq [0,1]^m $, and scale $\tilde f$ so that $|\tilde f|_\alpha=1 $. This rescaling will only effect the constants in Corollary \ref{cor:intrinsic}, and is what leads to the dependence of $P$ on $L_1,L_2 $. We then obtain that there exists a neural network $\psi$ of the advertised size, such that 
 $$|\tilde f \circ \tilde E^{-1}(y)-\psi(y)|\leq \epsilon, \quad \forall y\in K/G $$
It follows that for all $x\in K$,
\begin{align*}|f(x)-\psi \circ E(x)|&=|\tilde f([x])-\psi \circ \tilde E([x])|\\
&=|\tilde f\circ \tilde E^{-1} \circ \tilde E([x])-\psi \circ \tilde E([x])|\\
&\leq \epsilon
\end{align*}

\end{proof}

Theorem \ref{thm:bi_lip} can be applied to automatically  obtain approximation rates for any bi-Lipschitz model. For the sake of concreteness, we state one specific example: bi-Lipschitz invariants using max filters. 

Max filters were first suggested in \cite{cahill2024group}, and were later studied in several papers, examples including \cite{mixon2023max,max_mixon2025,max_QADDURA}.
They are defined for vector space $V=\RR^N$, acted on by  $G$,  a closed subgroup of $O(N)$ (that is, a subgroup of isometries without a translation component). For a given ``template'' vector $z\in \RR^N$, an invariant mapping $E(\cdot;z):V\to \RR $ is defined as 
$$E(x,z)=max_{g\in G} \langle gx,z \rangle . $$
Similarly, for a collection $Z=(z_1,\ldots,z_D)$ an invariant max filtering mapping $E(\cdot;Z):V\to \RR^D $ is defined via 
$$E(x;Z)=\left(E(x;z_1),\ldots, E(x;z_D) \right), j=1,\ldots,D.$$
In \cite{cahill2024group} it was shown that for any generic choice of $2N$ templates, the mapping $E(x;Z) $ in injective. Moreover, it was shown in \cite{balan2023g} that for every \emph{finite} group $G$, this mapping, when injective, will also be bi-Lipschitz. Thus, combining these results with Theorem \ref{thm:bi_lip} we obtain the following corollary 
\begin{cor}
Let $G$ be a finite subgroup of $O(N)$. Let $C_1>0$ and $\dint\geq 2 $, and let $K\subseteq \RR^N$ be a $G$ invariant subset, such that $K/G$ is a  compact subset of  $V/G$ of diameter $1$, and such that
$$N_r(K/G)\leq C_1r^{-\dint}, \quad \forall r\geq 0 .$$
Then there exists a choice of templates $Z=(z_1,\ldots,z_{2N})$, where each  $z_j$ is in $\RR^N $, such that, for all $\epsilon>0$, and every
$f:V \to \RR$  which is invariant and $\alpha$-Holder with $|f|_\alpha=1 $,  there exists a ReLU networks $\psi$ expressible with  $\leq  P(1/\epsilon)^{\frac{\dint}{2\alpha}}\log(1/\epsilon) $ parameters, such that 
$$|f(x)-\psi\circ E(x;Z)|<\epsilon, \quad  \forall x \in K .$$
The constant $P$ does not depend on $f$ or $\epsilon$, but does depend on $C_1,N,Z,\dint,\alpha,G$.  
\end{cor}
Similar corollaries can be obtained for other bi-Lipschitz models, such as all bi-Lipschitz models mentioned in the beginning of this section.

\textbf{Acknowledgements}
We are grateful to Thien Le for useful discussions related to this work. JWS and AS were supported by National Science Foundation (DMS-2424305). JWS was also supported by the MURI ONR grant N00014-20-1-2787. ND and SH were supported by ISF grant  272/23.

    \bibliography{refs.bib}
\appendix
\input{proofs.tex}

\end{document}

%% file: proofs.tex
\section{Proofs}
We prove Proposition \ref{prop:dim}:

\begin{proof}[Proof of Proposition \ref{prop:dim}]
We first consider the case of \textbf{permutations}. Let $K\subseteq \RR^{d\times n}$ such that $K/G\subseteq \RR^{d\times n}/S_n $ is a compact set. The function
$$F([X])=\|X\|_F $$
is Lipschitz on $\RR^{d\times n}$ and $S_n$ invariant, and therefore it is also Lipschitz
 on the quotient space by Lemma \ref{lem:holder}. Therefore, it attains a maximum $M$ on $K/G $. Accordingly, $K/G$ is contained in $q(B_M) $, where 
$$B_M=\{X\in \RR^{d\times n}| \quad \|X\|_F\leq M \} .$$
For an appropriate $C=C(M)$, we have
$$N_r(B_M)\leq C_Mr^{-nd}, \quad \forall r>0 ,$$
and since the quotient map $q$ is 1-Lipschitz, we can push any cover of $B_M$ of balls with radius $r$ to a cover of $q(B_M)$ by balls with the same radius. Thus 
$$N_r(K/G)\leq N_r(B_M/G)\leq N_r(B_M)
\leq C_Mr^{-nd}, \quad \forall r>0 . $$

Next we consider the combined action of \textbf{permutations and rigid motions}. Let $K\subseteq \RR^{d\times n} $ be a $G$ invariant set such that $K/G$ is compact. As in \eqref{eq:cent}, we denote 
\begin{equation*}
\cent(X)=X-\frac{1}{n}X1_n1_n^T. \end{equation*}
The  function  $F([X])=\|\cent(X)\|_F$ is Lipschitz on $\RR^{d\times n}$ and is $G$ invariant. Therefore it is  well defined and Lipschitz on the quotient space, by Lemma \ref{lem:holder}, and therefore it attains a maximum $M$ on $K/G$. It follows that  $K/G $ is contained in  $q(B_M)$. Next, we claim that $\RR^{d\times n}/G $ is  contained in $q(\M)$, where  
$$\mathcal{M}=\{X\in \RR^{d\times n}| X1_n=0_d, X[i,j]=0, \forall d\geq i>j \geq 1  \} .$$
This is because, given any $X\in \RR^{d\times n}$ we can (a)  centralize $X$, and then, denoting by $k$ the rank of $\cent(X)$ (b) relabel the entries of $\tilde X=\cent(X)$ so that $\mathrm{rank}(\tilde x_1,\ldots,\tilde x_k)=k $. We can then apply the Gram-Schmidt procedure to obtain a rotation $U$ such that $U\tilde X$ is in $\M$ (if $k<d$ this rotation still exists, but is not unique), and $U\tilde X\in[X] $. It follows that $K/G\subseteq q(B_M\cap \M) $
Under our assumption that $n\geq d+1 $, the dimension of the subspace $\M$ is 
$$\dim(\mathcal{M})=nd-\frac{d(d-1)}{2}-d=nd-\dim(E(d))=nd-\dim(G),$$
and therefore there exists some $C>0$ such that 
$$N_r(B_M\cap \M)\leq Cr^{-(nd-\dim(G))}, \quad \forall r>0. $$
Again using the fact that $q$ is $1$-Lipschitz, we obtain 
$$N_r(K/G)\leq N_r(q(B_M\cap \M))\leq N_r(B_M\cap \M) \leq Cr^{-(nd-\dim(G))}, \quad \forall r>0. $$

 \end{proof}